\documentclass{article}

\usepackage{fullpage}
\usepackage{algorithm}
\usepackage{algorithmic}
\usepackage{natbib}
\usepackage{microtype}
\usepackage{graphicx}
\usepackage{subfig}
\usepackage{booktabs} 
\usepackage{hyperref}

\usepackage{amsmath}
\usepackage{amssymb}
\usepackage{mathtools}
\usepackage{amsthm}
\usepackage[capitalize,noabbrev]{cleveref}

\theoremstyle{plain}
\newtheorem{theorem}{Theorem}[section]

\newtheorem{lemma}[theorem]{Lemma}

\theoremstyle{definition}
\newtheorem{definition}[theorem]{Definition}
\newtheorem{assumption}[theorem]{Assumption}
\theoremstyle{remark}
\newtheorem{remark}[theorem]{Remark}

\usepackage{graphicx}
\usepackage{amsmath,amsfonts,amssymb}
\usepackage{physics}
\usepackage{esvect}
\usepackage{cancel}
\usepackage{relsize}
\usepackage{xcolor}

\usepackage{todonotes}

\newcommand{\R}{\mathbb{R}}
\newcommand{\ours}{\text{PDP-OP}}

\DeclareMathOperator*{\argmin}{arg\,min}

\makeatletter
\newcommand{\multiline}[1]{%
    \begin{tabularx}{\dimexpr\linewidth-\ALG@thistlm}[t]{@{}X@{}}
        #1
    \end{tabularx}
}

\def\code#1{\texttt{#1}} 

\usepackage[nice]{nicefrac}
\usepackage{xfrac}
\usepackage[utf8]{inputenc}
\usepackage{hyperref}
\usepackage{xcolor}
\usepackage{url}
\usepackage{amsmath,amsthm,amsfonts,amssymb,bbm}
\usepackage{bm}
\usepackage{mathtools}
\usepackage{authblk}
\usepackage{mdframed}
\usepackage{lipsum}

\usepackage{wrapfig}
\usepackage{graphicx}
\graphicspath{ {./Figures/} }
\usepackage{caption}
\usepackage{subcaption}
\newcommand{\minus}{\scalebox{0.75}[1.0]{$-$}}

\author[1]{Krishna Acharya}
\author[2]{Franziska Boenisch}
\author[1]{Rakshit Naidu}
\author[1]{Juba Ziani}

\affil[1]{Georgia Institute of Technology}
\affil[2]{CISPA Helmholtz Center for Information Security}

\date{
    $^1$\{krishna.acharya, rakshitnaidu, jziani3\}@gatech.edu\\%
    $^2$boenisch@cispa.de\\[2ex]%
}
\title{Personalized Differential Privacy for Ridge Regression}

\begin{document}

\maketitle

\begin{abstract}
The increased application of machine learning (ML) in sensitive domains requires protecting the training data through privacy frameworks, such as differential privacy (DP). 
DP requires to specify a uniform privacy level $\varepsilon$ that expresses the maximum privacy loss that each data point in the entire dataset is willing to tolerate.
Yet, in practice, different data points often have different privacy requirements. Having to set one uniform privacy level is usually too restrictive, often forcing a learner to guarantee the stringent privacy requirement, at a large cost to accuracy. 
To overcome this limitation, we introduce our novel Personalized-DP Output Perturbation method (PDP-OP) that enables to train Ridge regression models with individual per data point privacy levels.
We provide rigorous privacy proofs for our PDP-OP as well as accuracy guarantees for the resulting model.
This work is the first to provide such theoretical accuracy guarantees when it comes to personalized DP in machine learning, whereas previous work only provided empirical evaluations. We empirically evaluate PDP-OP on synthetic and real datasets and with diverse privacy distributions. We show that by enabling each data point to specify their own privacy requirement, we can significantly improve the privacy-accuracy trade-offs in DP. We also show that PDP-OP outperforms the personalized privacy techniques of~\citet{jorgensen}.
\end{abstract}

\section{Introduction}

Over the last decade, the amount of private data collected about individuals has experienced an exponential growth. 
As the data is used for computing statistics, training recommender systems, and automated decision-making in sensitive domains, such as medicine, privacy concerns around this data are growing.
The gold standard for analyzing the data with privacy guarantees is Differential Privacy~(DP)~\cite{dp_og}.
DP allows to perform meaningful analyses of the entire dataset while protecting the privacy of individuals.
It guarantees that if any single individual in a dataset were to change their data point, the (distribution of) outcomes of the differentially private mechanism remains roughly the same.
The closeness of the outcomes is parametrized by a privacy parameter $\varepsilon$ that captures the level of privacy. 
This $\varepsilon$ represents the maximal privacy loss that any individual contributing data to the dataset is willing to accept, with small $\varepsilon$ indicating high levels of privacy.

However, DP comes with a major limitation: It requires to set the privacy level $\varepsilon$ uniformly for the entire dataset.
Implicitly, this suggests that all individuals whose data is present in the dataset have the same privacy requirements.
Yet, this is not accurate as individuals were shown to have diverse privacy requirements~\cite{jensen2005privacy,berendt2005privacy, acquistiindiv05}.
By setting a uniform privacy budget in DP, this budget must match the individual whose privacy requirements are the strongest. 
This means that $\varepsilon$ has to correspond to the highest privacy requirement in the dataset.
Thereby---since the implementation of DP usually relies on the addition of noise with higher privacy requiring higher amounts of noise being added---DP often yields unfavorable privacy-accuracy trade-offs~\cite{li2009tradeoff,tramer2020differentially,bagdasaryan2019differential}. 

In this paper, we argue standard DP is overly conservative and propose a new algorithm to train ridge regression with per-individual personalized privacy guarantees.\footnote{In the remainder of this work, we will, without loss of generality, assume that each data point is contributed by a different individual.}
We provide rigorous privacy proofs and accuracy guarantees for our algorithm, thereby, distinguishing ourselves from prior work on machine learning with personalized guarantees~\cite{hdp,jorgensen,boenisch2022individualized,boenisch2023have} that solely provides empirical evaluations. Our personalized privacy techniques also differ from that  of~\citet{boenisch2022individualized,boenisch2023have}, as we describe in more details in Section~\ref{sec:background}.  

In summary, we make the following contributions:
\begin{itemize}
    \item We propose the first personalized DP algorithm specialized to the case of Ridge regression in Section~\ref{sec:algos}, Algorithm~\eqref{Alg:output_boundedlabel}.
    \item \vspace{-1mm} We provide rigorous privacy proofs in Section~\ref{sec:privacy_guarantees} and accuracy guarantees in Section~\ref{sec:utility_guarantees}.
    \item \vspace{-1mm} We perform extensive empirical evaluations in Section~\ref{sec:experiments}. We highlight that i) our algorithm significantly outperforms standard output perturbation-based DP ridge regression in terms of privacy-accuracy trade-offs on multiple datasets and diverse privacy distributions in Section~\ref{sec:experiments_nonpersonalized} and Appendix~\ref{app:experiments_nonpers}. We also show that we outperform the personalized privacy technique of~\citet{jorgensen} in Section~\ref{sec:experiments_jorgensen} and Appendix~\ref{app:experiments_jorgensen}.
\end{itemize}

\section{Preliminaries and Related Work}\label{sec:background}

\paragraph{Differential Privacy.}
\label{sub:DP}
An algorithm $M$ satisfies $\pmb{(\varepsilon)}$\textbf{-DP}~\cite{dp_og}, if for any two datasets $D, D'\subseteq \mathcal{D}$ that differ in any one data point $x$, and any set of outputs~$R$
\begin{align}
\label{align:dp}
    \mathbb{P}\left[M(D) \in R\right] \leq e^\varepsilon \cdot \mathbb{P}\left[M(D') \in R\right].
\end{align}
This definition provides a theoretical upper bound on the influence that a single data point can have on the outcome of the analysis over the whole dataset.
The value $\varepsilon \in \mathbb{R}_+$ denotes the privacy level; with lower values corresponding to stronger privacy guarantees.

\paragraph{Personalized Privacy.}

Personalizing privacy guarantees are highly relevant, given that studies showed how society consists at least of three different groups of individuals, requiring strong, average, or weak privacy protection~\cite{jensen2005privacy, berendt2005privacy, acquistiindiv05}.
Without personalization, when applying DP to datasets that hold data from individuals with different privacy requirements, $\varepsilon$ needs to be set to the lowest $\varepsilon$ encountered among all individuals whose data we choose to use in our computation. This can often yield unfavorable privacy-utility trade-offs, due to having to throw away too much data or to have to use a stringent value of $\varepsilon$ for everyone. Instead, several previous works concurrently introduced the concept of \emph{personalized} DP. It was introduced formally in~\citet{hdp} and~\citet{jorgensen} in the context of central DP; it was also used informally by~\citet{cummings2015accuracy} in the context of mechanism design for data acquisition with local DP. Formally, the definition of personalized DP is the following:

\begin{definition}[$i$-neighboring]
 Two datasets $D$ and $D'$ are neighboring with respect to data point $i$ (or “i-neighbors”) if they differ only in data point $i$.
\end{definition}

\begin{definition}[Personalized DP]\label{def:personalized_dp}
A randomized algorithm $\mathcal{M}$ is $\varepsilon_i$-differentially private with respect to data point $i$, if for any outcome set $O \subset Range\left(\mathcal{M}\right)$ and for all $i$-neighboring databases $D, D'$
\[
\Pr \left[\mathcal{M}(D) \in O \right] \leq \exp(\varepsilon_i) \Pr \left[\mathcal{M}(D') \in O \right].
\]
\end{definition}

Two of the most prevalent techniques for personalized DP are personalized \textit{data sampling} and \textit{sensitivity pre-processing}. 

Both methods change how much the output of a computation depends on a particular data point: the lesser the dependency on a given data point, the more privacy this point gets.
The idea of data sampling for personalized privacy was introduced by~\citet{jorgensen}, and later used in the works of~\citet{niu2021adapdp,boenisch2022individualized,boenisch2023have}. 
Data sampling introduces randomness in the dataset: each data point can be sub-sampled or up-sampled before being fed into a standard DP algorithm. 

Sensitivity pre-processing, which is the approach used in this work, in contrast, can be implemented deterministically. It modifies the query of interest to have different sensitivities for different data points, where sensitivity is a standard notion in DP on how much a computation can change across neighbouring datasets.  

Sensitivity pre-processing for personalized privacy was originally introduced by~\citet{hdp} through linear pre-processing or ``stretching'' of the input data; a caveat of this method is that it requires strong assumptions on how changing the data changes per-user sensitivity.
A general-purpose method for manipulating the sensitivity of a query while maintaining accuracy is provided by~\citet{cummings_preprocessing}, but is unfortunately NP-hard to implement in the general case and is constructive, rather than given in closed form. Specializations of this method for the case of moment estimation have been recently used in~\citet{fallah_ec2022,ziani_satml2023}: both papers rely on weighted moment estimation, where the weight is lower for data points with stronger privacy requirements. A difficulty with such re-weightings is that they often need to be tailored to the specific learning task at hand. This is the approach we take and challenge we face in this work. Finally, \citet{partitioning} proposed two partitioning algorithms that first separate the data in different groups according to privacy requirements, and then process these groups separately. This approach was shown sub-optimal for learning-based applications~\cite{boenisch2022individualized}.

\paragraph{Private Empirical Risk Minimization.} There has been a significant line of work on making linear regression, generalized linear models, and empirical risk minimization (ERM) differentially-private~\citep{chaudhuri2008privacy,chaudhuri2011differentially,kifer2012private,bassily2014private, jain2014near,wang2018revisiting,bassily2019private,chen2020understanding,song2020characterizing,cai2021cost,song2021evading,alabi2022differentially,arora2022differentially}. These papers focus on standard DP, as opposed to personalized DP, and span a relatively large number of different techniques. Most relevant to us are the initial works in this space by~\citet{chaudhuri2008privacy,chaudhuri2011differentially}. In particular, they analyze obtaining DP in regression and ERM through an \emph{output perturbation} technique: ERM is first performed non-privately, then noise is added directly to the non-private estimator. We similarly rely on output perturbation in this paper, noting that this is a good starting point to the study of personalized DP in the context of regression. Most recently, output perturbation saw a new analysis for generalized linear models by~\citet{arora2022differentially}.

\section{Algorithms and Guarantees for Personalized Privacy in Ridge Regression}\label{sec:algos}

\paragraph{Our Setup.}
Consider a dataset $\mathcal{D} = \{ (x_i, y_i) \in \mathcal{X}  \times \mathcal{Y} : i = 1, 2, \cdots, n\}$ consisting of a total $n$ data points.
We assume that features are bounded; without loss of generality, we work with $x_i \in [0,1]^d$ for all $i \in [n]$.
We also assume that the labels are bounded, and w.l.o.g. set $y_i \in [-1,1]$ for all $i \in [n]$. Beyond this, each data point $i \in [n]$ has a \emph{privacy requirement}, in the form of a DP parameter $\varepsilon_i > 0$. The lower the value of $\varepsilon_i$, the more stringent the privacy requirement of data point $i$, as per Definition~\ref{def:personalized_dp}.

Our main focus is Ridge regression. I.e., we are aiming to find a linear model $x^\top \bar{\theta}$, parametrized by $\bar{\theta}$, that predicts the labels as accurately as possible. Our goal is to find the $\bar\theta$ that minimizes the Ridge loss
\[
L(\theta,\lambda) 
= \frac{1}{n} \sum_{i=1}^n \left(y_i - \theta^\top x_i \right)^2 
+ \lambda \| \theta \|^2_2.
\]
However, we cannot release $\bar{\theta}$, as it encodes information about the dataset ${(x_i,y_i)}$. Instead, we provide an estimator $\hat{\theta}$ whose performance on our Ridge loss is good, while at the same time ensuring that we satisfy DP with parameter $\varepsilon_i$ for all data points $i \in [n]$ \emph{simultaneously}.

\paragraph{Our Main Algorithm.} The main idea of our algorithm, ``Personalized-Differentially-Private Output Perturbation'' (or ``PDP-OP'')  is as follows: if we wanted to obtain traditional, non-personalized DP, we could first compute a non-private estimate $\bar{\theta}$, then add well-chosen noise $Z$ for privacy with density $\nu(b) \propto \exp(-\eta \norm{b}_2)$. This takes inspiration from~\citet{chaudhuri2008privacy, chaudhuri2011differentially}; however, our analysis exhibits differences due to differences in our estimators to incorporate personalized DP. Here, to instead obtain personalized DP, we rely on an idea studied by~\citet{cummings_preprocessing}: manipulating or pre-processing the sensitivity (i.e., how much a given data point can change the outcome) of our query with respect to each data point $i$. Our techniques for sensitivity pre-processing are very different from that of~\citet{cummings_preprocessing}: their method algorithmically constructs a ``good'' query with the desired sensitivities (but is intractable in the general case), while we obtain such a query efficiently and in closed-form through carefully manipulating the weight given to each data point in our loss function.

More precisely, our sensitivity pre-processing technique re-weights each data point $i$ by a weight $w_i$. The smaller the weight $w_i$, the smaller the impact the $i$-th data point has on the output model. Intuitively, this means that lower $w_i$'s correspond to less information encoded in the output $\hat{\theta}$ about data point $i$, and better privacy for $i$. When giving different $w_i$ to different data points, we can then ensure different, personalized privacy levels for different data points. Our algorithm is given formally in Algorithm~\ref{Alg:output_boundedlabel}. We then show formally in Section~\ref{sec:privacy_guarantees} how to choose the weights $w_i$ and the noise parameter $\eta$ in order to guarantee personalized DP with respect to privacy preferences $\varepsilon_1, \ldots, \varepsilon_n$.

\begin{algorithm}[H]
\caption{Personalized-Differentially-Private Output Perturbation (PDP-OP) \label{Alg:output_boundedlabel}}

\textbf{Inputs}: Dataset $\mathcal{D} = \{\left(x_i, y_i\right)~\text{for}~i \in [n]\}$; weights vector $\textbf{w} \geq 0$ with $\sum_{i=1}^n w_i = 1$; noise parameter $\eta$. \\
\textbf{Output}: Private estimator $\hat{\theta}$

\begin{algorithmic}[1]
\STATE Compute non-private estimate $\bar{\theta}$ as follows:
\begin{align}\label{eq:weightedLS}
\bar{\theta} 
= \arg\min_{\theta} \sum_{i=1}^n w_i \left(y_i - \theta^\top x_i \right)^2 
+ \lambda \| \theta \|^2_2.
\end{align}
\STATE Sample $Z$ as a random variable with probability density function $\propto \exp(-\eta \norm{b}_2)$
\STATE Return private estimate $\hat{\theta} = \bar{\theta} + Z$
\end{algorithmic}
\end{algorithm}

\begin{remark}[How to sample $Z$]\label{rmk:sampling}
It is known that for probability density function $\nu(b) \propto \exp(-\eta \norm{b}_2)$ for $Z$: 
\begin{itemize}
\item $\Vert Z \Vert_2$ follows a Gamma$(\alpha,\beta)$ distribution with $\alpha = d$ and $\beta = \eta$, which has density $f(r) \propto r^{d-1} \exp{-\eta r}. $\footnote{An easy way to see this informally is the following: the total mass on $\Vert Z \Vert_2 = r$ is proportional to the total mass on all $Z$'s with norm $r$---which is proportional to $\frac{2 \pi^{d/2}}{\Gamma(d/2)} r^{d-1}$ (the surface area of the $d$-dimensional $\ell_2$-sphere)---multiplied by the mass on any single $Z$ of norm $r$---which is proportional to $\exp{-\eta r}$. A more formal proof, omitted here, consists in writing the integration for $P[\Vert Z \Vert_2 \leq R]$, then doing a change of variable to hyper-spherical coordinates in the corresponding integral.}
\item For any given value $r \triangleq \Vert Z \Vert_2$, $Z$ is uniform on the $\ell_2$-sphere of radius $r$. This can be seen immediately as the density only depends on $\Vert Z \Vert$, so all realizations $z$ with the same norm have the same density.
\end{itemize}
In turn, to sample $Z$, it suffices to i) sample the radius $R$ from Gamma$(d,\eta)$, then ii) sample $Y$ uniformly at random from the $\ell_2$-sphere of radius $1$, and set $Z = R Y$. 
\end{remark}

\subsection{Privacy Guarantees}\label{sec:privacy_guarantees}
We now state the personalized privacy guarantee obtained by our algorithm. We remind the reader that this privacy guarantee relies on the observations being normalized such that $x_i \in [0,1]^d$ and $y_i \in [-1,1]$ for all $i \in [n]$:
\begin{theorem}\label{thm:privacy_guarantee_boundedlabel}
Fix privacy specifications $\varepsilon_1, \ldots, \varepsilon_n > 0$. Let $B(\lambda) = \min \left(\frac{1}{\sqrt{\lambda}}, \frac{\sqrt{d}}{\lambda}\right)$. Algorithm~\ref{Alg:output_boundedlabel} with parameters $w_i = \frac{\varepsilon_i}{\sum_{j=1}^n \varepsilon_j}$ for all $i$ and $\eta = \frac{\lambda}{2\sqrt{d} \left(1 + \sqrt{d} B(\lambda)\right)}\sum_{j=1}^n \varepsilon_j$ is $\varepsilon_i$-personalized differentially private with respect to data point $i$ for every $i \in [n]$.
\end{theorem}

Additionally, we show another version of our privacy guarantee that uses additional assumptions on the data and on the best regression parameter absent regularization, when such information is available:
\begin{assumption}[Bounded $\bar{\theta}$]\label{as:boundedtheta}
Let $\bar{\theta}_0 = \argmin_\theta \sum_{i=1}^n w_i \left(y_i - \theta^\top x_i \right)^2$, when $\lambda = 0$. There is a known constant $B$ such that $\norm{\bar{\theta}_0}_2 \leq B$.
\end{assumption}

We incorporate such boundedness assumptions as they have also been used in previous work, such as~\cite{wang2018revisiting,arora2022differentially}.
Our privacy guarantee, under additional Assumption~\ref{as:boundedtheta}, is given by:

\begin{theorem}\label{thm:privacy_guarantee_boundedtheta}
Suppose Assumption~\ref{as:boundedtheta} holds. Fix privacy specifications $\varepsilon_1, \ldots, \varepsilon_n > 0$. Algorithm~\ref{Alg:output_boundedlabel} with parameters $w_i = \frac{\varepsilon_i}{\sum_{j=1}^n \varepsilon_j}$ for all $i$ and $\eta = \frac{\lambda}{2 \sqrt{d} (B\sqrt{d}  + 1)}\sum_{j=1}^n \varepsilon_j$ is $\varepsilon_i$-personalized differentially private with respect to data point $i$ for every $i \in [n]$.
\end{theorem}
We can easily see that if $B \leq B(\lambda) = \min \left(\frac{1}{\sqrt{\lambda}},\frac{\sqrt{d}}{\lambda}\right)$, this bound adds noise $Z$ with a bigger parameter $\eta$ compared to in Theorem~\ref{thm:privacy_guarantee_boundedlabel}, which corresponds to adding \emph{less} noise. This will lead to better accuracy guarantees in regimes in which we put little weight on the regularization parameter, provided that we know or can estimate such a bound $B$.

\begin{proof}[Proof Sketch] The full proof is given in Appendix~\ref{app:proof_privacy}. The idea of the proof is similar to that of the output perturbation technique of \cite{chaudhuri2011differentially}, adapted to Ridge regression and personalized DP. A major difference comes in the fact that we use different weights $w_i$ for different data points. We examine the first order condition satisfied by $\bar{\theta}$ and note that $\bar{\theta}$ has a bounded gradient difference across two $i$-neighbouring datasets. However, unlike \citet{chaudhuri2011differentially}, the gradient difference here is bounded by a term that is proportional to $w_i$ when the databases are neighbouring in data point $i$. Combined with the fact that the loss function is strongly convex with strong convexity parameter $2 \lambda$, $\bar{\theta}$ cannot change too much across two neighbouring databases by a standard argument and bound its $\ell_2$-sensitivity by a term proportional to $w_i/\lambda$. Adding well-tailored noise for privacy concludes the proof.
\end{proof}

\subsection{Accuracy Guarantees}\label{sec:utility_guarantees}
We now provide theoretical bounds on the accuracy of our framework. We note that we are the first to provide such theoretical accuracy bounds for ridge regression with personalized DP. The framework of~\citet{jorgensen} is relatively general, but said generality prevents them from obtaining worst-case theoretical accuracy bounds, and they focus on an empirical evaluation of the performance (in terms of loss or accuracy) of their sampling framework.

We make the assumption that the label generating process is in fact approximately linear, which is the main use case in which linear regression should be used in the first place. This assumption is common in the literature; it is found for example in the seminal high-dimensional probability book of~\cite{vershynin2018high}. Importantly, note that this assumption is only made in order to characterize the theoretical accuracy of our framework. Our personalized DP bounds of Section~\ref{sec:privacy_guarantees} crucially \textit{do not} rely on Assumption~\ref{as:linearDGP}. 

\begin{assumption}\label{as:linearDGP}
Given a feature vector $x_i$, the label $y_i$ is given by $y_i = x_i^\top \theta^* + Z_i$ where $\theta^* \in \mathbb{R}^d$ and the $Z_i$'s are independent and identically distributed Gaussian variables with mean $0$ and standard deviation $\sigma > 0$.
\end{assumption}

We now provide a bound on how well we recover $\theta^*$, the true data generating process, as closely as possible as a function of our dataset and our choice of privacy parameters. We bound the distance between our estimate $\hat{\theta}$ and the true data generating parameter $\theta^*$ below:

\begin{theorem}[Accuracy of $\hat{\theta}$]\label{thm:accuracy}
Let $\eta$ be chosen as per Theorems~\ref{thm:privacy_guarantee_boundedlabel} absent assumptions, and Theorem~\ref{thm:privacy_guarantee_boundedtheta} under Assumption~\ref{as:boundedtheta}. For any $\delta > 0$, with probability at least $1-\delta$, we have that for any $\lambda > 0$, 
\begin{align*}
\Vert \theta^* - \hat{\theta} \Vert_2 
\leq \frac{\norm{\theta^*} }{1 + \frac{\lambda_{min} \left(\sum_{i=1}^n w_i x_i x_i^\top\right)}{\lambda}} 
+ \frac{1}{\eta} \left(d + \sqrt{\frac{2d}{\delta}} \right) 
 + \frac{\sigma}{\lambda} \sqrt{\frac{2d}{\delta}} \Vert \vec{w} \Vert
\end{align*}
for all $\lambda > 0$, where $\vec{w} = (w_1,\ldots,w_n)$. 
\end{theorem}

\begin{proof}
The proof of Theorem~\ref{thm:accuracy} is in Appendix~\ref{app:proof_accuracy}.
\end{proof}

Our bound starts with a bias term: $\norm{\theta^*}  \big/\left(1 + \lambda_{min} \left(\sum_{i=1}^n w_i x_i x_i^\top\right)/\lambda \right)$. This term controls the bias due to Ridge regression itself, independently of the noise added for privacy and from the noise in the labels, and is unavoidable due to the nature of Ridge regression. First, in the worst case in which $\lambda$ grows larger and larger, Ridge regression recovers $\hat{\theta} \to 0$, and the bias due to regression alone, absent noise and privacy, would converge to $\theta^*$; this is exactly what our bound shows. Second, when $\lambda_{min} \left(\sum_{i=1}^n w_i x_i x_i^\top\right) = 0$, observations are not full-rank, and bias is unavoidable due to the fact that the unregularized regression problem has several solutions, some of which need not be close to $\theta^*$. However, when $\lambda_{min} \left(\sum_{i=1}^n w_i x_i x_i^\top\right)$ becomes larger, the regression problem absent regularization becomes better conditioned and easier, and the impact of the Ridge regression parameter $\lambda$ on the recovered non-private $\bar{\theta}$ is reduced. This is reflected in our expression, as we see that when $\lambda_{min} \left(\sum_{i=1}^n w_i x_i x_i^\top\right)/\lambda$ increases (i.e. $\lambda$ decreases), this bias term starts becoming smaller. The bias term disappears completely in this case as $\lambda \to 0$; this is to be expected, as in case we are running unregularized regression under full rank observations, in which case we know the minimizer is both unique and unbiased.

The second term, $\frac{1}{\eta} \left(d + \sqrt{\frac{2d}{\delta}}\right)$, is due to the noise added for privacy with parameter $\eta$. Importantly, note that this $\eta$ is proportional to $\sum_{j=1}^n \varepsilon_j$, giving us an accuracy bound where the privacy dependency that evolves with $1/\sum_{j=1}^n \varepsilon_j$. Absent personalization, instead, we would have to use a common, stringent privacy level of $\min_j \varepsilon_j$ for all agents $i$; to so, one would typically weight all data points equally and add noise with parameter scaling in $\eta \sim n \min_j \varepsilon_j$ to guarantee privacy under output perturbation. This would lead to a bound evolving in $1/n \min_j \varepsilon_j$. The accuracy bound for personalized privacy could then be significantly worse when some agents have very stringent privacy levels: i.e., when $\min_j \varepsilon_j$ small, but at the same time the average-case privacy level $\frac{1}{n} \sum_j \varepsilon_j$ is more lenient. We show in Section~\ref{sec:experiments} that this effect is not only observed in toy examples and in our theoretical accuracy bound, but also consistently across our experimental results.

Further, in the second term, $\eta$---as seen in Theorems~\ref{thm:privacy_guarantee_boundedlabel} and~\ref{thm:privacy_guarantee_boundedtheta}---is an increasing function of $\lambda$, and the second term is therefore decreasing in $\lambda$. This is in contrast with the bias term, which increases when $\lambda$ increase. The idea is that increasing $\lambda$ increases the weight on the regularization hence increases the bias of our estimator; however, at the same time, it decreases the amount of noise we need to add for privacy (bigger $\eta$ means less noise), leading to a smaller variance term due to the addition of differential privacy. One can in practice pick $\lambda$ to trade-off these two terms, both experimentally or based on our theoretical accuracy bound.

Finally, the third term $\frac{\sigma}{\lambda} \sqrt{\frac{2d}{\delta}} \norm{\vec{w}}$ is due to the fact that labels are noisy hence there are inaccuracies introduced due to the fact that linear regression is not a perfectly accurate model for our setting. As $\sigma$ decreases, the label noise also decreases, the observations become closer to forming a line, and this term grows smaller. The bigger the value of $\lambda$, the less the regression estimator depends on the data, and the lesser the impact of this label noise. Finally, the term has a dependency on the weights $\vec{w}$: the smaller the norm of $\vec{w}$, the better the weights are spread out (with the norm being minimized when $w_1 = \ldots = w_n = 1/n$) across data points, and the better the accuracy as we incorporate information about more and more of the data in our estimator.

\section{Experiments}\label{sec:experiments}
In this section, we evaluate the performance of our algorithm experimentally. Importantly, we highlight that our goal \emph{is not} to evaluate the performance of ``output perturbation'' (which first computes a non-private estimator then adds noise for privacy) for private regression, as this has been done extensively in previous work~\cite{chaudhuri2008privacy,chaudhuri2011differentially}. Rather, we highlight the performance of our re-weighting technique, in particular compared to the absence of data reweighting (which gives the same level of privacy to all data points) and to the sampling-based technique for personalized privacy of~\citet{jorgensen}. For this reason and to isolate the performance of our re-weighting technique versus the sampling of \citet{jorgensen}, we fix our experimental evaluation to have all baselines be based on output perturbation techniques\footnote{The code is available at \url{https://github.com/krishnacharya/pp-ridge}}.

We divide this section into two parts: i) we show how much the addition of personalized privacy improves accuracy compared to non-personalized privacy; ii) we compare our results to~\cite{jorgensen}, and show improvement both in terms of accuracy and variance of the estimator. 

\paragraph{Choice of Privacy Budgets.} To validate our personalized privacy setting, we follow a similar segregation scheme as~\cite{hdp,jorgensen}. We categorize data points into 3 segments, in order of most stringent to less stringent privacy requirements: \emph{conservatives} (high privacy), \emph{mediums} or \emph{pragmatists} (medium privacy), and \emph{liberals} (low privacy). The fraction of conservatives, mediums, and liberals in the population is denoted by $f_c$, $f_m$ and $f_l = 1 - (f_c + f_m)$ respectively. Each segment has their own privacy parameter denoted respectively $\varepsilon_c$, $\varepsilon_m$ and $\varepsilon_l$, where $\varepsilon_c < \varepsilon_m < \varepsilon_l$ (remember that lower $\varepsilon$, means stronger privacy requirement). In our experiments, we assign the personalized privacy budgets to the conservatives and mediums by uniformly sampling from the ranges $[\varepsilon_c, \varepsilon_m], [\varepsilon_m, \varepsilon_l]$ respectively. As for the liberals, they all receive the same single highest privacy budget $\varepsilon_l$. This follows~\citet{utilawareniu2020,jorgensen}.

As default values, we set $f_c=0.34, f_m=0.43, f_l=0.23$, $\varepsilon_c = 0.01, \varepsilon_m = 0.2, \varepsilon_l = 1.0$ unless otherwise specified. 
We provide experiments in Appendix~\ref{app:experiments} where we show how our insights extend as we change these privacy parameters.
\paragraph{Synthetic Data Generation.}
For the synthetic data, we draw $\theta^*$ uniformly over the unit sphere $S^{d-1}$.  Each feature $x$ is drawn uniformly over support $[0,1]^d$, and its corresponding label $y = \frac{x^\top \theta^*}{\sqrt{d}}$: the label then satisfies $|y| \leq \frac{1}{\sqrt{d}} \norm{x}_2 \norm{\theta^*}_2 \leq 1$. We consider perfect linear relationships between $x$ and $y$ with no noise in the labels (i.e., $\sigma = 0$ as per the notations of Section~\ref{sec:utility_guarantees}); we do so to decouple the effect of linear models being an imperfect hypothesis class from the performance of our method. We rely on our real dataset described below, both in the main body and in Appendix~\ref{app:experiments}, for situations in which the relationship between features and labels can only approximately be captured by a linear model.

\paragraph{Real Dataset.} For our experiments on real data, we use the ``Medical Cost'' dataset \cite{kaggleMedicalCost} which looks at an individual medical cost prediction task. Each individual's features are split into three numeric \{\code{age, BMI, \#children}\} and three categorical features \{\code{sex, smoker, region}\}. The dataset also has a real-valued medical \code{charges} column that we use as our label. 

{\bf Data pre-processing.}
We use min-max scaling to normalize the numeric features as well as the label to the range $[0,1]$. For any categorical features, we use standard one-hot-encoding. To deal with affine rather than simply linear relationships in the data, we add a $d+1$-th feature to each feature vector $x$, that corresponds to our model's intercept.

{\bf Metrics.} 
We evaluate the performance of  $\hat{\theta}$ on a held-out test set of size $N_{test}$, using the following metrics : (1)~\textbf{Unregularized test loss:}
$
\sum_{i=1}^{N_{test}} \frac{1}{N_{test}} (y_i - x_i^\top \hat{\theta})^2 
$ and (2)~\textbf{Regularized test loss:} 
$
\sum_{i=1}^{N_{test}} \frac{1}{N_{test}} (y_i -x_i^\top \hat{\theta})^2 + \lambda \norm{\hat{\theta}}_2^2\text{.}
$ 
The $\lambda$ we use in our evaluation is the same $\lambda$ that we use in our training loss and in Algorithm~\ref{Alg:output_boundedlabel}.

\subsection{Improvements over standard Differential Privacy}\label{sec:experiments_nonpersonalized}

Our first results show the improvements in accuracy when we use personalized DP as opposed to standard (or non-personalized) DP. To do so, we compare to what we call the \emph{non-personalized} baseline which provides the same privacy level $\varepsilon$ to all data points. We let $\varepsilon^* = \min_i \varepsilon_i$ is chosen to satisfy the most stringent privacy preferences (remember that smaller $\varepsilon$ means more privacy) among all data points. To implement this baseline, we simply use Algorithm~\ref{Alg:output_boundedlabel}, but with the privacy preference profile being $(\varepsilon^*,\ldots,\varepsilon^*)$. Note that the weights are all the same and equal to $1/n$ and the added noise scales as a function of $\varepsilon^*$. Hence our baseline implementation just follows the standard regression algorithm with output perturbation of~\citet{chaudhuri2008privacy} and~\citet{chaudhuri2011differentially}.

Table~\ref{tab:syn_4_1_plevel_34_43_23} shows the performance of our algorithm versus the non-personalized baseline across varying regularization $\lambda$, while fixing the other parameters $f_c, f_m, \varepsilon_c,~\varepsilon_m$. We note that we get consistent improvements of several order of magnitudes across all values of $\lambda$. The improvement is, for example, roughly of an order of magnitude of $100$ when it comes to both unregularized and regularized mean-squared error on the test set. This shows that leveraging differing privacy preferences across differing data points can lead to huge improvements in terms of privacy-accuracy trade-offs for differential privacy. Table \ref{tab:insurance_4_1_plevel_34_43_23} shows that significant improvements also occur on the real dataset.

\begin{table}[!h]
\centering
\resizebox{0.9\textwidth}{!}{%
\begin{tabular}{|c | cc|cc|}
\hline
\begin{tabular}[c]{@{}c@{}}Regularization\\ parameter\\ Lambda ($\lambda$)\end{tabular} &
  \begin{tabular}[c]{@{}c@{}}Unregularized\\ test loss \\ (\ours)\end{tabular} &
  \begin{tabular}[c]{@{}c@{}}Unregularized\\ test loss \\ (non-personalized)\end{tabular} &
  \begin{tabular}[c]{@{}c@{}}Regularized\\ test loss \\ (\ours)\end{tabular} &
  \begin{tabular}[c]{@{}c@{}}Regularized\\ test loss\\ (non-personalized)\end{tabular} \\
\hline
1.00   & $\mathbf{8.54 \cross 10^{2}}$  & $4.60 \cross 10^{5}$  & $\mathbf{3.32 \cross 10^{3}}$  & $1.77 \cross 10^{6}$ \\
3.00   & $\mathbf{3.88 \cross 10^{1}}$  & $2.08 \cross 10^{4}$  & $\mathbf{3.82 \cross 10^{2}}$  & $2.03 \cross 10^{5}$ \\
5.00   & $\mathbf{9.78}$                  & $5.20 \cross 10^{3}$  & $\mathbf{1.50 \cross 10^{2}}$  & $8.08 \cross 10^{4}$ \\
7.00   & $\mathbf{3.81}$                  & $2.06 \cross 10^{3}$  & $\mathbf{8.36 \cross 10^{1}}$  & $4.46 \cross 10^{4}$ \\
10.00  & $\mathbf{1.49}$                  & $8.18 \cross 10^{2}$  & $\mathbf{4.58 \cross 10^{1}}$  & $2.45 \cross 10^{4}$ \\
15.00  & $\mathbf{5.30 \cross 10^{-1}}$ & $2.75 \cross 10^{2}$  & $\mathbf{2.34 \cross 10^{1}}$  & $1.26 \cross 10^{4}$ \\
20.00  & $\mathbf{2.54 \cross 10^{-1}}$ & $1.29 \cross 10^{2}$  & $\mathbf{1.49 \cross 10^{1}}$  & $8.01 \cross 10^{3}$ \\
25.00  & $\mathbf{1.44 \cross 10^{-1}}$ & $7.45 \cross 10^{1}$  & $\mathbf{1.07 \cross 10^{1}}$  & $5.63 \cross 10^{3}$ \\
50.00  & $\mathbf{2.28 \cross 10^{-2}}$ & $1.11 \cross 10^{1}$  & $\mathbf{3.06}$                  & $1.64 \cross 10^{3}$ \\
75.00  & $\mathbf{9.35 \cross 10^{-3}}$ & 3.78                  & $\mathbf{1.56}$                  & $8.29 \cross 10^{2}$ \\
100.00 & $\mathbf{5.82 \cross 10^{-3}}$ & 1.80                  & $\mathbf{1.01}$                  & $5.36 \cross 10^{2}$ \\
125.00 & $\mathbf{4.39 \cross 10^{-3}}$ & 1.06                  & $\mathbf{7.37 \cross 10^{-1}}$ & $3.91 \cross 10^{2}$ \\
150.00 & $\mathbf{3.68 \cross 10^{-3}}$ & $6.79 \cross 10^{-1}$ & $\mathbf{5.73 \cross 10^{-1}}$ & $3.05 \cross 10^{2}$ \\
175.00 & $\mathbf{3.33 \cross 10^{-3}}$ & $4.76 \cross 10^{-1}$ & $\mathbf{4.67 \cross 10^{-1}}$ & $2.49 \cross 10^{2}$ \\
200.00 & $\mathbf{3.09 \cross 10^{-3}}$ & $3.64 \cross 10^{-1}$ & $\mathbf{3.97 \cross 10^{-1}}$ & $2.10 \cross 10^{2}$ \\
\hline
\end{tabular}
}
\caption{Loss of \ours~compared to standard DP on the synthetic dataset with $d = 30, n = 100$, keeping $\varepsilon_c = 0.01, \varepsilon_m = 0.2, \varepsilon_l = 1.0, f_c = 0.34, f_m = 0.43, f_l = 0.23$.}
\label{tab:syn_4_1_plevel_34_43_23}
\end{table}

\begin{table}[!h]
\centering
\resizebox{0.9\textwidth}{!}{%
\begin{tabular}{|c|cc|cc|}
\hline
\begin{tabular}[c]{@{}c@{}}Regularization\\ parameter\\ Lambda ($\lambda$)\end{tabular} &
  \begin{tabular}[c]{@{}c@{}}Unregularized\\ test loss\\ (PDP-OP)\end{tabular} &
  \begin{tabular}[c]{@{}c@{}}Unregularized\\ test loss\\ (non-personalized)\end{tabular} &
  \begin{tabular}[c]{@{}c@{}}Regularized\\ test loss\\ (PDP-OP)\end{tabular} &
  \begin{tabular}[c]{@{}c@{}}Regularized\\ test loss\\ (non-personalized)\end{tabular} \\
\hline
0.01 & $\mathbf{1.19 \cross 10^5}$    & $2.18 \cross 10^8$   & $\mathbf{1.22 \cross 10^5}$    & $2.31 \cross 10^8$ \\
0.05 & $\mathbf{1.03 \cross 10^3}$    & $1.89 \cross 10^6$   & $\mathbf{1.16 \cross 10^{3}}$  & $2.12 \cross 10^6$ \\
0.10 & $\mathbf{1.34 \cross 10^2}$    & $2.49 \cross 10^{5}$ & $\mathbf{1.70 \cross 10^{2}}$  & $3.08 \cross 10^5$ \\
0.50 & $\mathbf{1.30}$                & $2.40 \cross 10^3$   & $\mathbf{3.03}$                & $5.52 \cross 10^3$ \\
0.60 & $\mathbf{8.10 \cross 10^{-1}}$ & $1.47 \cross 10^3$   & $\mathbf{2.02}$                & $3.65 \cross 10^3$ \\
0.70 & $\mathbf{5.29 \cross 10^{-1}}$ & $9.26 \cross 10^2$   & $\mathbf{1.47}$                & $2.61 \cross 10^3$ \\
0.80 & $\mathbf{3.78 \cross 10^{-1}}$ & $6.36 \cross 10^2$   & $\mathbf{1.11}$                & $1.98 \cross 10^3$ \\
0.90 & $\mathbf{2.78 \cross 10^{-1}}$ & $4.59 \cross 10^2$   & $\mathbf{8.79 \cross 10^{-1}}$ & $1.54 \cross 10^3$ \\
1.00 & $\mathbf{2.15 \cross 10^{-1}}$ & $3.45 \cross 10^{2}$ & $\mathbf{7.12 \cross 10^{-1}}$ & $1.24 \cross 10^3$ \\
2.00 & $\mathbf{6.80 \cross 10^{-2}}$ & $5.18 \cross 10^1$   & $\mathbf{2.26 \cross 10^{-1}}$ & $3.19 \cross 10^2$ \\
3.00 & $\mathbf{5.52 \cross 10^{-2}}$ & $1.73 \cross 10^1$   & $\mathbf{1.39 \cross 10^{-1}}$ & $1.52 \cross 10^2$ \\
5.00 & $\mathbf{5.54 \cross 10^{-2}}$ & 4.49                 & $\mathbf{9.65 \cross 10^{-2}}$ & $6.25 \cross 10^1$ \\
\hline
\end{tabular}%
}
\caption{Loss of \ours~compared to standard DP on the Medical cost dataset, keeping $\varepsilon_c = 0.01, \varepsilon_m = 0.2, \varepsilon_l=1.0, f_c = 0.34, f_m = 0.43, f_l = 0.23$.}
\label{tab:insurance_4_1_plevel_34_43_23}
\end{table}

Appendix~\ref{app:experiments_nonpers} provides more experiments where we vary parameters $f_c,~\varepsilon_c,~\varepsilon_m,~n$ on the synthetic and real datasets.

\subsection{Comparison to \citet{jorgensen}}\label{sec:experiments_jorgensen}

\begin{table*}[!h]
\centering
\resizebox{0.9\textwidth}{!}{%
\begin{tabular}{|c | cccc|cccc |}
\hline
\begin{tabular}[c]{@{}c@{}}Regularization\\ parameter\\ Lambda ($\lambda$)\end{tabular} &
  \begin{tabular}[c]{@{}c@{}}Unregularized\\ test loss \\ (\ours)\end{tabular} &
  \begin{tabular}[c]{@{}c@{}}Unregularized\\ test loss\\ (Jorgensen max)\end{tabular} &
  \begin{tabular}[c]{@{}c@{}}Unregularized\\ test loss\\ (Jorgensen mean)\end{tabular} &
  \begin{tabular}[c]{@{}c@{}}Unregularized\\ non-private \\ test loss\end{tabular} &
  \begin{tabular}[c]{@{}c@{}}Regularized\\ test loss \\ (\ours)\end{tabular} &
  \begin{tabular}[c]{@{}c@{}}Regularized\\ test loss\\ (Jorgensen max)\end{tabular} &
  \begin{tabular}[c]{@{}c@{}}Regularized\\ test loss\\ (Jorgensen mean)\end{tabular} &
  \begin{tabular}[c]{@{}c@{}}Regularized\\ non-private \\ test loss\end{tabular} \\
\hline
1.00 &
  $\mathbf{8.54 \cross 10^{2}}$ &
  $1.09 \cross 10^{3}$ &
  $1.89 \cross 10^{3}$ &
  $5.73 \cross 10^{-32}$ &
  $\mathbf{3.32 \cross 10^{3}}$ &
  $4.16 \cross 10^{3}$ &
  $7.55 \cross 10^{3}$ &
  $2.01 \cross 10^{-3}$ \\
3.00 &
  $\mathbf{3.88 \cross 10^{1}}$ &
  $4.82 \cross 10^{1}$ &
  $8.64 \cross 10^{1}$ &
  $5.73 \cross 10^{-32}$ &
  $\mathbf{3.82 \cross 10^{2}}$ &
  $4.79 \cross 10^{2}$ &
  $8.64 \cross 10^{2}$ &
  $2.11 \cross 10^{-3}$ \\
5.00 &
  $\mathbf{9.78}$ &
  $1.21 \cross 10^{1}$ &
  $2.15 \cross 10^{1}$ &
  $5.73 \cross 10^{-32}$ &
  $\mathbf{1.50 \cross 10^{2}}$ &
  $1.89 \cross 10^{2}$ &
  $3.40 \cross 10^{2}$ &
  $2.17 \cross 10^{-3}$ \\
7.00 &
  $\mathbf{3.81}$ &
  4.81 &
  8.85 &
  $5.73 \cross 10^{-32}$ &
  $\mathbf{8.36 \cross 10^{1}}$ &
  $1.05 \cross 10^{2}$ &
  $1.89 \cross 10^{2}$ &
  $2.20 \cross 10^{-3}$ \\
10.00 &
  $\mathbf{1.49}$ &
  1.88 &
  3.44 &
  $5.73 \cross 10^{-32}$ &
  $\mathbf{4.58 \cross 10^{1}}$ &
  $5.75 \cross 10^{1}$ &
  $1.03 \cross 10^{2}$ &
  $2.24 \cross 10^{-3}$ \\
15.00 &
  $\mathbf{5.30 \cross 10^{-1}}$ &
  $6.61 \cross 10^{-1}$ &
  1.16 &
  $5.73 \cross 10^{-32}$ &
  $\mathbf{2.34 \cross 10^{1}}$ &
  $2.97 \cross 10^{1}$ &
  $5.33\cross 10^{1}$ &
  $2.28 \cross 10^{-3}$ \\
20.00 &
  $\mathbf{2.54 \cross 10^{-1}}$ &
  $3.12 \cross 10^{-1}$ &
  $5.64 \cross 10^{-1}$ &
  $5.73 \cross 10^{-32}$ &
  $\mathbf{1.49 \cross 10^{1}}$ &
  $1.87 \cross 10^{1}$ &
  $3.39 \cross 10^{1}$ &
  $2.31 \cross 10^{-3}$ \\
25.00 &
  $\mathbf{1.44 \cross 10^{-1}}$ &
  $1.81 \cross 10^{-2}$ &
  $3.27 \cross 10^{-1}$ &
  $5.73 \cross 10^{-32}$ &
  $\mathbf{1.07 \cross 10^{1}}$ &
  $1.31 \cross 10^{1}$ &
  $2.39 \cross 10^{1}$ &
  $2.33 \cross 10^{-3}$ \\
50.00 &
  $\mathbf{2.28 \cross 10^{-2}}$ &
  $2.82 \cross 10^{-2}$ &
  $4.92 \cross 10^{-2}$ &
  $5.73 \cross 10^{-32}$ &
  $\mathbf{3.06}$ &
  3.83 &
  6.95 &
  $2.37 \cross 10^{-3}$ \\
75.00 &
  $\mathbf{9.35 \cross 10^{-3}}$ &
  $1.12 \cross 10^{-2}$ &
  $1.81 \cross 10^{-2}$ &
  $5.73 \cross 10^{-32}$ &
  $\mathbf{1.56}$ &
  1.95 &
  3.56 &
  $2.39 \cross 10^{-3}$ \\
100.00 &
  $\mathbf{5.82 \cross 10^{-3}}$ &
  $6.61 \cross 10^{-3}$ &
  $1.01 \cross 10^{-2}$ &
  $5.73 \cross 10^{-32}$ &
  $\mathbf{1.01}$ &
  1.26 &
  2.28 &
  $2.40 \cross 10^{-3}$ \\
200.00 &
  $\mathbf{3.09 \cross 10^{-3}}$ &
  $3.25 \cross 10^{-3}$ &
  $3.92 \cross 10^{-3}$ &
  $5.73 \cross 10^{-32}$ &
  $\mathbf{3.97 \cross 10^{-1}}$ &
  $4.97 \cross 10^{-1}$ &
  $9.02 \cross 10^{-1}$ &
  $2.42 \cross 10^{-3}$ \\
300.00 &
  $\mathbf{2.68 \cross 10^{-3}}$ &
  $2.76 \cross 10^{-3}$ &
  $3.04 \cross 10^{-3}$ &
  $5.73 \cross 10^{-32}$ &
  $\mathbf{2.42 \cross 10^{-1}}$ &
  $3.02 \cross 10^{-1}$ &
  $5.46 \cross 10^{-1}$ &
  $2.42 \cross 10^{-3}$ \\
400.00 &
  $\mathbf{2.56 \cross 10^{-3}}$ &
  $2.60 \cross 10^{-3}$ &
  $2.76 \cross 10^{-3}$ &
  $5.73 \cross 10^{-32}$ &
  $\mathbf{1.75 \cross 10^{-1}}$ &
  $2.16 \cross 10^{-1}$ &
  $3.92 \cross 10^{-1}$ &
  $2.42 \cross 10^{-3}$ \\
500.00 &
  $\mathbf{2.51 \cross 10^{-3}}$ &
  $2.54 \cross 10^{-3}$ &
  $2.62 \cross 10^{-3}$ &
  $5.73 \cross 10^{-32}$ &
  $\mathbf{1.36 \cross 10^{-1}}$ &
  $1.70 \cross 10^{-1}$ &
  $3.06 \cross 10^{-1}$ &
  $2.43 \cross 10^{-3}$ \\
\hline
\end{tabular}%
}
\caption{Lower loss compared to \citet{jorgensen} on the synthetic dataset with $d = 30,~n = 100$, keeping $\varepsilon_c = 0.01, \varepsilon_m = 0.2, \varepsilon_l=1.0, f_c = 0.34, f_m = 0.43, f_l = 0.23$.}
\label{tab:acc_syn_4_2_1_plevel_34_43_23}
\end{table*}

\begin{table*}[!h]
\centering
\resizebox{0.9\textwidth}{!}{
\begin{tabular}{|c|cccc|cccc|}
\hline
\begin{tabular}[c]{@{}c@{}}Regularization\\ parameter\\ Lambda ($\lambda$)\end{tabular} &
  \begin{tabular}[c]{@{}c@{}}Unregularized\\ test loss\\ (PDP-OP)\end{tabular} &
  \begin{tabular}[c]{@{}c@{}}Unregularized\\ test loss\\ (Jorgensen max)\end{tabular} &
  \begin{tabular}[c]{@{}c@{}}Unregularized\\ test loss\\ (Jorgensen mean)\end{tabular} &
  \begin{tabular}[c]{@{}c@{}}Unregularized\\ non-private\\ test loss\end{tabular} &
  \begin{tabular}[c]{@{}c@{}}Regularized\\ test loss\\ (PDP-OP)\end{tabular} &
  \begin{tabular}[c]{@{}c@{}}Regularized\\ test loss\\ (Jorgensen max)\end{tabular} &
  \begin{tabular}[c]{@{}c@{}}Regularized\\ test loss\\ (Jorgensen mean)\end{tabular} &
  \begin{tabular}[c]{@{}c@{}}Regularized\\ non-private\\ test loss\end{tabular} \\
\hline
0.01 &
  $\mathbf{1.19 \cross 10^5}$ &
  $1.50 \cross 10^5$ &
  $2.87 \cross 10^5$ &
  $9.43 \cross 10^{-3}$ &
  $\mathbf{1.22 \cross 10^5}$ &
  $1.52 \cross 10^5$ &
  $2.93 \cross 10^5$ &
  $1.08 \cross 10^{-2}$ \\
0.05 &
  $\mathbf{1.03 \cross 10^3}$ &
  $1.26 \cross 10^3$ &
  $2.43 \cross 10^3$ &
  $9.43 \cross 10^{-3}$ &
  $\mathbf{1.16 \cross 10^3}$ &
  $1.42 \cross 10^3$ &
  $2.75 \cross 10^3$ &
  $1.51 \cross 10^{-2}$ \\
0.10 &
  $\mathbf{1.34 \cross 10^2}$ &
  $1.66 \cross 10^2$ &
  $3.14 \cross 10^2$ &
  $9.43 \cross 10^{-3}$ &
  $\mathbf{1.70 \cross 10^2}$ &
  $2.11 \cross 10^2$ &
  $4.00 \cross 10^2$ &
  $1.90 \cross 10^{-2}$ \\
0.50 &
  $\mathbf{1.30}$ &
  $1.67$ &
  3.10 &
  $9.43 \cross 10^{-3}$ &
  $\mathbf{3.03}$ &
  3.80 &
  7.19 &
  $3.46 \cross 10^{-2}$ \\
0.60 &
  $\mathbf{8.10 \cross 10^{-1}}$ &
  $9.89 \cross 10^{-1}$ &
  1.85 &
  $9.43 \cross 10^{-3}$ &
  $\mathbf{2.02}$ &
  2.50 &
  4.76 &
  $3.67 \cross 10^{-2}$ \\
0.70 &
  $\mathbf{5.29 \cross 10^{-1}}$ &
  $6.55 \cross 10^{-1}$ &
  1.23 &
  $9.43 \cross 10^{-3}$ &
  $\mathbf{1.47}$ &
  1.80 &
  3.43 &
  $3.86 \cross 10^{-2}$ \\
0.80 &
  $\mathbf{3.78 \cross 10^{-1}}$ &
  $4.65 \cross 10^{-1}$ &
  $8.71 \cross 10^{-1}$ &
  $9.43 \cross 10^{-3}$ &
  $\mathbf{1.11}$ &
  1.39 &
  2.63 &
  $4.02 \cross 10^{-2}$ \\
0.90 &
  $\mathbf{2.78 \cross 10^{-1}}$ &
  $3.42 \cross 10^{-1}$ &
  $6.30 \cross 10^{-1}$ &
  $9.43 \cross 10^{-3}$ &
  $\mathbf{8.79 \cross 10^{-1}}$ &
  1.08 &
  2.03 &
  $4.17 \cross 10^{-2}$ \\
1.00 &
  $\mathbf{2.15 \cross 10^{-1}}$ &
  $2.61 \cross 10^{-1}$ &
  $4.76 \cross 10^{-1}$ &
  $9.43 \cross 10^{-3}$ &
  $\mathbf{7.12 \cross 10^{-1}}$ &
  $8.70 \cross 10^{-1}$ &
  1.62 &
  $4.30 \cross 10^{-2}$ \\
2.00 &
  $\mathbf{6.80 \cross 10^{-2}}$ &
  $7.53 \cross 10^{-2}$ &
  $1.06 \cross 10^{-1}$ &
  $9.43 \cross 10^{-3}$ &
  $\mathbf{2.26 \cross 10^{-1}}$ &
  $2.65 \cross 10^{-1}$ &
  $4.59 \cross 10^{-1}$ &
  $5.19 \cross 10^{-2}$ \\
3.00 &
  $\mathbf{5.52 \cross 10^{-2}}$ &
  $5.72 \cross 10^{-2}$ &
  $6.87 \cross 10^{-2}$ &
  $9.43 \cross 10^{-3}$ &
  $\mathbf{1.39 \cross 10^{-1}}$ &
  $1.59 \cross 10^{-1}$ &
  $2.52 \cross 10^{-1}$ &
  $5.69 \cross 10^{-2}$ \\
5.00 &
  $\mathbf{5.54 \cross 10^{-2}}$ &
  $5.61 \cross 10^{-2}$ &
  $5.89 \cross 10^{-2}$ &
  $9.43 \cross 10^{-3}$ &
  $\mathbf{9.65 \cross 10^{-2}}$ &
  $1.05 \cross 10^{-1}$ &
  $1.43 \cross 10^{-1}$ &
  $6.27 \cross 10^{-2}$ \\
\hline
\end{tabular}
}
\caption{Lower loss compared to \citet{jorgensen} on the MedicalCost dataset, keeping $\varepsilon_c = 0.01, \varepsilon_m = 0.2, \varepsilon_l=1.0, f_c = 0.34, f_m = 0.43, f_l = 0.23$. }
\label{tab:insurance_4_2_1_plevel_34_43_23}
\end{table*}

\begin{figure*}[!h]
\vspace{-6mm}
\centering
    \subfloat[Unregularized Loss ($d=30, n=100$)]{
        \centering
        \includegraphics[width=0.37\textwidth]{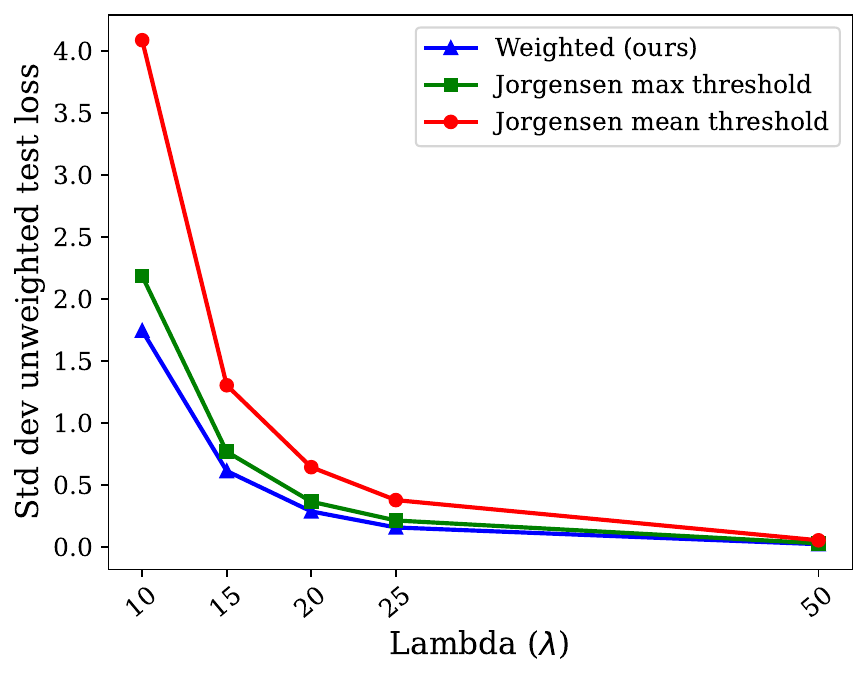}
    }
    \hspace{15mm}
    \subfloat[Regularized Loss ($d=30, n=100$)]{
        \includegraphics[width=0.37\textwidth]{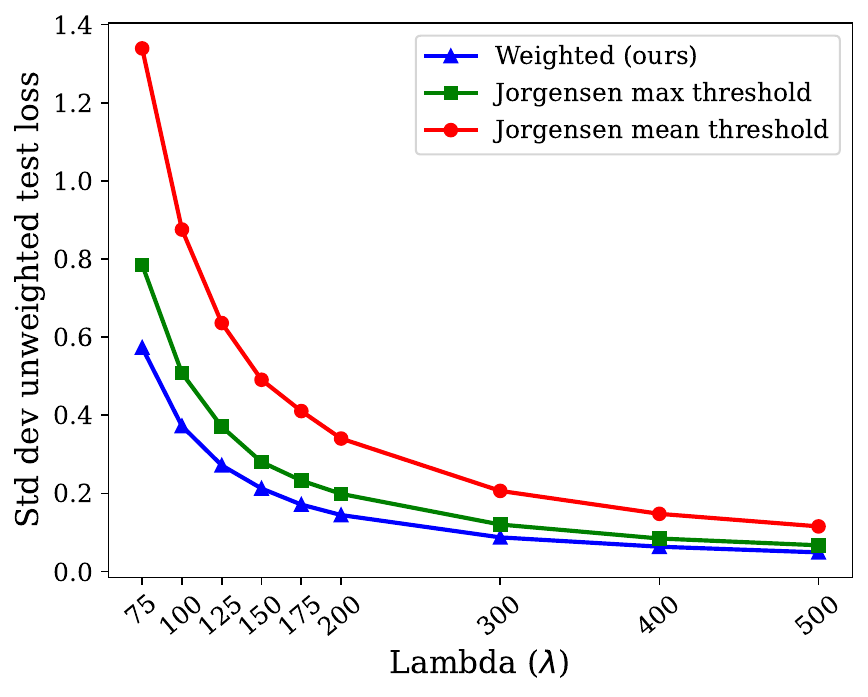}
    }
    \caption{Lower standard deviation compared to~\cite{jorgensen} on the synthetic dataset, while varying the regularization parameter $\lambda$, keeping $\varepsilon_c = 0.01, \varepsilon_m = 0.2, \varepsilon_l=1.0, f_c = 0.34, f_m = 0.43, f_l = 0.23$. 
}\label{fig:syn_jorgensen_lambda_std}
\end{figure*}

\begin{figure*}[!h]
\centering
    \subfloat[Unregularized Loss]{
        \centering
        \includegraphics[width=0.37\textwidth]{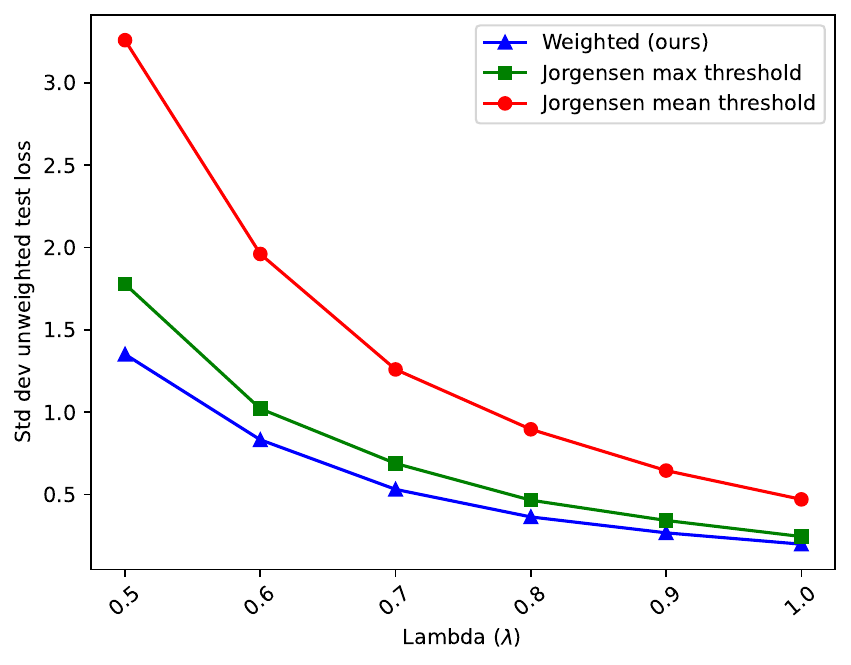}
    }
    \hspace{15mm}
    \subfloat[Regularized Loss]{
        \includegraphics[width=0.37\textwidth]{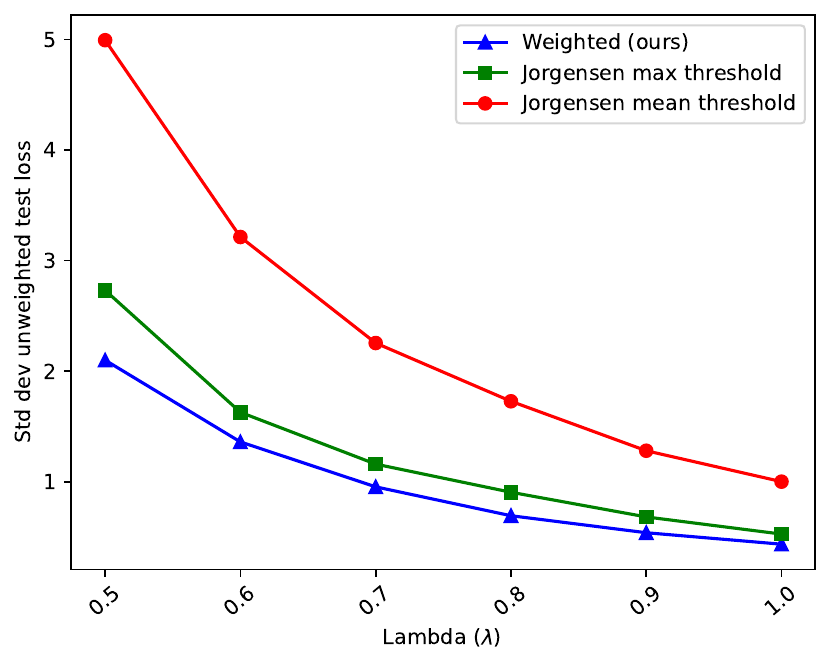}
    }
    \caption{Lower standard deviation compared to~\cite{jorgensen} on the Medical cost dataset, while varying the regularization parameter $\lambda$, keeping $\varepsilon_c = 0.01, \varepsilon_m = 0.2,\varepsilon_l=1.0, f_c = 0.34, f_m = 0.43, f_l = 0.23$.  
}\label{fig:insurance_jorgensen_lambda_std}
\end{figure*}

\paragraph{Experimental setup.}
We compare our approach to that of~\citet{jorgensen}, who proposed the first algorithm for personalized differential privacy in the central privacy model. Their approach is the following: first, they pick a threshold $t$. Then, they sample each data point $i$ in $D$ with probability $\frac{\exp{\varepsilon_i}-1}{\exp{t}-1}$ if $\varepsilon_i < t$, and probability $1$ otherwise. Lastly, they run a standard non-personalized algorithm that is $t$-differentially private. We fix the non-personalized algorithm to follow the output perturbation technique described in~\citet{chaudhuri2008privacy,chaudhuri2011differentially} to keep comparisons apple-to-apple and only compare the impact of our re-weighting versus the thresholding then sub-sampling approach of~\citet{jorgensen}. We implement the non-personalized estimator as per the baseline described in Section~\ref{sec:experiments_nonpersonalized}. For the choice of $t$, we try out both $t = \max_i \varepsilon_i$ (that we refer to as ``Jorgensen max'' or ``max threshold'') and $t = \frac{1}{n} \sum_i \varepsilon_i$ (that refer to as ``Jorgensen mean'' or ``mean threshold''). Both these choices of $t$ were proposed by~\citet{jorgensen} itself.

In the rest of this section, we present our experimental results. We note that our framework consistently leads to improved performance over \cite{jorgensen}. This is seen in two ways: first, our unregularized and regularized losses are consistently lower than those of Jorgensen for the vast majority of choices of instance parameters and of regularization parameter $\lambda$. This show that our re-weighting method is consistently more accurate compared to the subsampling method of~\citet{jorgensen} when it comes to accuracy. Beyond this, we also note that our results are more consistent: the standard deviation of the loss of our technique is also lower than that of \citet{jorgensen}. I.e., our results are more consistent across different runs of the algorithms and different realizations of the noise. 

Intuitively, one of the advantages of our technique over that of \citet{jorgensen} is that by re-weighting instead of sampling, we do not discard any of our dataset; we believe this is one potential source of improvement. Another, more subtle reason, may be that our framework adds all noise $Z$ centrally, at the end of the computation, while \citet{jorgensen} adds noise both locally (when sub-sampling data points) and centrally. It is well understood that adding noise centrally (i.e., within the computation) in differential privacy leads to better privacy-accuracy trade-offs than adding noise locally (i.e., at the level of each data point), which may be another reason for our improvements.

\paragraph{Improvements in loss.} Table~\ref{tab:acc_syn_4_2_1_plevel_34_43_23} provides a snapshot of our performance versus that of the baseline of~\citet{jorgensen}, for our default choice of privacy parameters. We perform consistently better than \citet{jorgensen} across both our metrics (unregularized and regularized loss), with improvements in loss of up to roughly $20$ percent under the max threshold. \citet{jorgensen}'s results when using the mean threshold instead are consistently worse. 
Table \ref{tab:insurance_4_2_1_plevel_34_43_23} shows that similar insights hold on our real dataset.
We provide additional experiments where we change the parameters of the problem such as $f_c,~\varepsilon_c,~\varepsilon_m,~n$ and when we consider our real dataset in Appendix~\ref{app:experiments_jorgensen}. 

\paragraph{Improvements in variability of the results.}

Figure~\ref{fig:syn_jorgensen_lambda_std} and ~\ref{fig:insurance_jorgensen_lambda_std} show the standard deviation of our loss compared to that of~\citet{jorgensen}, estimated across $10,000$ runs for each technique on the synthetic and real datasets respectively. The figure clearly highlights how our method exhibits less variability, leading to more consistent results across different runs and realizations of the noise.

 We also provide additional experiments where we change the parameters of the problem such as $f_c,~\varepsilon_c,~\varepsilon_m,~n$ and when we consider our real dataset in Appendix~\ref{app:experiments_jorgensen}, and note there that our insights still hold across these experiments.

\section{Conclusion and Future Work}
We proposed a new algorithm, Personalized-DP Output Perturbation (PDP-OP), which allows to train Ridge regression models with individual per-data point privacy requirements. 
We formally prove \ours's~personalized privacy guarantees and provide rigorous and theoretical results for the accuracy guarantees of our framework. We are in fact the first to provide a theoretical accuracy guarantee for personalized-DP methods in machine learning, to the best of our knowledge.
Our empirical evaluation on synthetic and real datasets highlights that PDP-OP significantly outperforms non-personalized DP, highlighting the need for personalized DP to vastly improve privacy-accuracy trade-offs in private ML. We also show that we outperform previous techniques for personalized DP, showing the advantages of using re-weighting over sub-sampling techniques.

The current paper aims to provide initial algorithms for personalized DP tailored to the case of Ridge regression. We chose to use \emph{output perturbation} as a simple starting point to provide initial insights and algorithms into personalized DP, and the benefits of data re-weighting over data sub-sampling. We, however, believe that there is still some leeway to improve the privacy-accuracy trade-offs of linear regression with personalized DP. In future work, we will incorporate our re-weighting technique with more advanced techniques for private regression, such as \emph{objective perturbation}, \emph{summary statistic perturbation}, or \emph{private gradient descent} (see related work for more details and examples).

\bibliography{sample}
\bibliographystyle{plainnat}


\newpage
\appendix
\onecolumn


\section{Proof of Main Results}
\subsection{Full Proof of Theorems~\ref{thm:privacy_guarantee_boundedlabel} and~\ref{thm:privacy_guarantee_boundedtheta}}\label{app:proof_privacy}

\paragraph{Preliminaries: Bound on $\Vert \bar{\theta} \Vert$ for Algorithm~\ref{Alg:output_boundedlabel} \label{Prelim-boundthetabar}} First, we show that the norm of $\bar{\theta}$ is bounded for Algorithm~\ref{Alg:output_boundedlabel}. This bound will be useful in bounding the gradient difference across two neighbouring databases in the proof of the privacy guarantee of Algorithm~\ref{Alg:output_boundedlabel}.

\begin{lemma}
\label{lem:weightedLSoptnorm}
    For any $\sum_{i=1}^n w_i = 1, w_i \geq 0$, the unconstrained minimizer of \eqref{eq:weightedLS}, $\bar{\theta}$, satisfies: 
    \[
    \norm{\bar{\theta}}_2 \leq \frac{1}{\sqrt{\lambda}}.
    \]
\end{lemma}
\begin{proof}
Let the weighted loss be defined as:
\[
L_w(\theta) \triangleq \sum_{i=1}^n w_i \left(y_i - \theta^\top x_i \right)^2 
+ \lambda \| \theta \|^2_2.
\]
For any $\theta \in \R^d$, $L_w(\bar{\theta}) \leq L_w(\theta)$. Therefore,
\[
\lambda \norm{\bar\theta}^2 \leq L_w(\bar{\theta}) \leq L_w(0) = \sum_{i=1}^n w_i y_i^2 \leq 1.
\]
\end{proof}

\begin{lemma}
\label{lem:weightedLSoptnorm}
    For any $\sum_{i=1}^n w_i = 1, w_i \geq 0$, the unconstrained minimizer of \eqref{eq:weightedLS}, $\bar\theta$, satisfies $\norm{\bar\theta}_2 \leq \frac{\sqrt{d}}{\lambda}$.
\end{lemma}

\begin{proof}
    We first find the closed form for $\bar{\theta}$. $\bar{\theta}$ is simply the unique solution to the unconstrained minimization problem 
\[
\argmin_{\theta \in \R^d}\sum_{i=1}^n w_i \left(y_i - \theta^\top x_i \right)^2 + \lambda \| \theta \|^2_2.
\]
Therefore, 
\[
\bar{\theta} = \argmin_{\theta \in \R^d}\sum_{i=1}^n w_i \left(y_i - \theta^\top x_i \right)^2 + \lambda \| \theta \|^2_2.
\]
Taking the first order-condition, we note that we must have 
\[
 -2 \sum_{i=1}^n w_i x_i \left(y_i - x_i^\top \bar{\theta}\right) + 2 \lambda \theta = 0.
\]
This can be rewritten as 
\[
2 \left(\sum_{i=1}^n w_i x_i x_i^\top + \lambda I \right) \bar{\theta} = 2 \sum_{i=1}^n w_i x_i y_i.
\]
Because $\lambda > 0$ and $\sum_{i=1}^n w_i x_i x_i^\top$ is positive semi-definite, $\sum_{i=1}^n w_i x_i x_i^\top + \lambda I$ is invertible, and we obtain the following closed-form expression for $\bar{\theta}$:
\[
\bar{\theta} = \left(\sum_{i=1}^n w_i x_i x_i^\top + \lambda I \right)^{-1} \sum_{i=1}^n w_i x_i y_i.
\] 

Taking the $\ell_2$-norm of both sides, we obtain, letting $\lambda_{min}(M),~\lambda_{max} (M)$ respectively denote the lowest and highest eigenvalues of any given matrix $M$:
\begin{align}
    \norm{\bar{\theta}}_2 &\leq 
    \lambda_{max} \left(\left(\lambda I + \sum_{i=1}^n w_i x_i x_i^\top \right)^{-1}\right) \norm{\sum_{i=1}^n w_i y_i x_i}_2 \label{eq_operator norm} \\
    &\leq \lambda_{max} \left(\left(\lambda I + \sum_{i=1}^n w_i x_i x_i^\top \right)^{-1}\right) \cdot \sqrt{d} \label{eq_labelvecsumnorm is 1}\\
    &\leq \frac{1}{\lambda_{min} \left(\lambda I + \sum_{i=1}^n w_i x_i x_i^\top\right)} \cdot \sqrt{d}\\
    &\leq \frac{\sqrt{d}}{\lambda},
\end{align}
where the second inequality follows from $\norm{x_i}_2 \leq \sqrt{d}$
\end{proof}


Under additional Assumption~\ref{as:boundedtheta}, note that $\norm{\bar\theta} \leq \norm{\bar{\theta}_0}$: indeed, if not, $\bar{\theta}_0$ is a better solution to the Ridge problem with parameter $\lambda$, noting that it has both i) lower---and in fact optimal---unregularized loss and ii) lower $\ell_2$-penalization. Then, we have $\Vert \bar\theta \Vert \leq B$. In the rest of the proof, we let $B(\lambda) \triangleq \min \left(\frac{1}{\sqrt{\lambda}},\frac{\sqrt{d}}{\lambda}\right)$ absent assumptions, and  $B(\lambda) \triangleq B$ under Assumption~\ref{as:boundedtheta}. Now, in both cases, we have $\norm{\bar\theta} \leq B(\lambda)$.




\paragraph{Sensitivity analysis of $\bar{\theta}$ through strong convexity and bounded gradients} We start with the following lemma that will help us bound the gradient different over two minimization problem: one over loss function $G(\theta)$ and one over modified loss function $G(\theta) +g(\theta)$, corresponding to the losses on neighboring databases $X'$ and $X$. The lemma is a slight modification from Lemma 7 of~\cite{chaudhuri2011differentially} that works with the gradient of $g(\theta)$ evaluated at a specific well-chosen point, instead of the maximum norm for the gradient of $g(\theta)$ over $\mathbb{R}^d$

  \begin{lemma}
 \label{Lemma:SC-outputpertubation}
     Let $G(\theta)$ and $g(\theta)$ be two vector valued, continuous, differentiable functions with $G(\theta)$ and $G(\theta) + g(\theta)$ both $\gamma$-strongly convex. Let $\theta_1 = \argmin_\theta G(\theta) + g(\theta)$, $\theta_2 = \argmin_\theta G(\theta)$, then
     \[ 
     \norm{\theta_1 - \theta_2}_2 \leq \frac{1}{\gamma} \norm{\nabla g(\theta_1)}_2.
     \]
 \end{lemma}
 
\begin{proof}
Note that $\theta_1$ and $\theta_2$ must satisfy the first order conditions, i.e. 
\begin{align}\label{eq:foc}
\nabla G(\theta_2) = 0 = \nabla G(\theta_1) + \nabla g(\theta_1).
\end{align}
Further, by strong convexity of $G$ with parameter $\gamma$, we have 
\begin{align*}
\gamma \norm{\theta_1-\theta_2}^2 
&\leq \left(\nabla G(\theta_2) - \nabla G(\theta_1)\right)^\top (\theta_2 - \theta_1) 
\\&= \nabla g(\theta_1)^\top (\theta_2 - \theta_1)
\\&\leq \norm{\nabla g(\theta_1)} \cdot \norm{\theta_1-\theta_2},
\end{align*}
where line 2 follows from Equation~\eqref{eq:foc} and line 3 follows from Cauchy-Schwarz.
\end{proof}

This allows us to bound the sensitivity with respect to agent $i$ of the non-private Ridge regression minimizer $\bar{\theta}$ as a function of $w_i$. Note that we define the $\ell_2$ sensitivity of $\bar{\theta}$ with respect to agent $i$ as
\[
\Delta_i \bar{\theta} = \max_{D,D'~\text{i-neighboring}} \Vert \bar\theta(D) - \bar{\theta}(D') \Vert_2,
\]
where $\bar{\theta}(D) = \argmin_{\theta} \sum_{i=1}^n w_i (y_i - \theta^\top x_i)^2 + \lambda \Vert \theta \Vert_2^2$. The sensitivity bound is then given by:

 \begin{theorem}
     Let $\bar{\theta} = \argmin_{\theta \in \R^d}\sum_{i=1}^n w_i \left(y_i - \theta^\top x_i \right)^2 + \lambda \| \theta \|^2_2$. $\bar{\theta}$ has $\ell_2$-sensitivity $(\Delta \theta)_i$ with respect to agent $i$ that satisfies $\Delta_i \bar{\theta}\leq \frac{2 \sqrt{d} w_i}{\lambda} \left(\sqrt{d} B(\lambda) + 1 \right)$.
 \end{theorem}

 \begin{proof}
Consider two neighbouring databases $X = ((x_1,y_1),\ldots, (x_n,y_n))$ and $X = ((x_1,y_1),\ldots, (x_i',y_i'), \ldots, (x_n,y_n))$ that only differ in the data of agent $i$.
Let $G(\theta) + g(\theta) = \sum_{j = 1}^n w_j \left(y_j - \theta^{\top} x_j \right)^2 + \lambda \norm{\theta}_2^2$, and $g(\theta) = w_i \left(y_i - \theta^{\top} x_i \right)^2- w_i \left(y_i' - \theta^{\top} x_i' \right)^2$.

First, we note that 
\[
G(\theta) + g(\theta) = \sum_{j} w_j \left(y_j - \theta^{\top} x_j \right)^2 + \lambda \norm{\theta}_2^2
\]
is the Ridge loss on database $D$, while 
\begin{align*}
G(\theta)  
&= \sum_{j} w_j \left(y_j - \theta^{\top} x_j \right)^2 + \lambda \norm{\theta}_2^2 - g(\theta) 
\\&= \sum_{j} w_j \left(y_j - \theta^{\top} x_j \right)^2 - w_i \left(y_i - \theta^{\top} x_i \right)^2 +  w_i \left(y_i' - \theta^{\top} x_i' \right)^2 + \lambda \norm{\theta}_2^2 
\\&=\sum_{j \neq i} w_j \left(y_j - \theta^{\top} x_j \right)^2 +  w_i \left(y_i' - \theta^{\top} x_i' \right)^2 + \lambda \norm{\theta}_2^2.
\end{align*}
is the loss on database $D'$. Further, both objectives are $2 \lambda$-strongly convex due to $\ell_2$-norm penalty term. Now, we have that $\nabla g(\theta) = 2w_i \left((\theta^\top x_i - y_i) x_i - (\theta^\top x_i' - y_i') x_i'\right) = 2w_i \left(\theta^\top x_i x_i - \theta^\top x_i' x_i' - y_i x_i + y_i' x_i'\right)$. Since $|y| \leq 1$ per our model, we have
\begin{align}
 \norm{\nabla g(\theta)}_2 \leq 2w_i (|\theta^\top x_i'| \Vert x_i' \Vert + |\theta^\top x_i| \Vert x_i \Vert + 2 \sqrt{d}). \nonumber
\end{align}
By the preliminaries section of this proof, we have that $\bar\theta(D)$, the minimizer of $G(\theta) + g(\theta)$, has norm at most $B(\lambda)$. On top of this, per our model, we have $\norm{x} \leq \sqrt{d}$ (since $x \in [0,1]^d)$. Therefore, $|\bar\theta(D)^\top x_i|, |\bar\theta(D')^\top x_i'| \leq \sqrt{d} B(\lambda)$ by Cauchy–Schwarz, and we have:

\begin{align}
\norm{\nabla g(\bar\theta)}_2
  \leq 2 w_i\left(2 d B(\lambda) + 2 \sqrt{d}\right). \label{eq:equivalentdomain}
\end{align}

 From lemma~\ref{Lemma:SC-outputpertubation}, we have that

\[
\Delta_i \bar{\theta}\leq \frac{1}{2 \lambda} \norm{\nabla g(\bar\theta)}_2.
\]
which becomes $\Delta_i \bar{\theta}\leq \frac{2 w_i \sqrt{d}}{\lambda} \left( \sqrt{d} B(\lambda) + 1 \right)$.
 \end{proof}

\paragraph{Implications for privacy} 

\begin{lemma}
    Algorithms~\ref{Alg:output_boundedlabel} is $\varepsilon_i$-differentially private for agent $i$ for all $i$ with 
    \[
    \varepsilon_i = \frac{2 w_i \eta \sqrt{d}}{\lambda} (\sqrt{d} B(\lambda)+1).
    \]
\end{lemma}

\begin{proof}
Let $D, D'$ be two datasets differing only in agent $i$'s data. Let us call out mechanism $M$. Let $L(\theta,D)$ be the Ridge regression loss on database $D$ evaluated at $\theta$, and let $\bar{\theta}(D) = \argmin_\theta L(\theta, D)$. For any given outcome $o$, we have that 
\[
\frac{P\left[M(D) = o\right]}{P\left[M(D') = o\right]} 
= \frac{P\left[\bar{\theta}(D) + Z = o\right]}{P\left[\bar{\theta}(D') + Z = o\right]} 
= \frac{P\left[Z = o - \bar{\theta}(D)\right]}{P\left[Z = o - \bar{\theta}(D')  \right]}.
\]
Noting that the probability density function is proportional to $f(z) \propto \exp \left(-\eta \Vert z \Vert_2\right)$, this can be written as 
\begin{align*}
\frac{P\left[M(D) = o\right]}{P\left[M(D') = o\right]} 
&= \exp \left(-\eta \norm{o -\bar{\theta}(D)} + \eta \norm{o -  \bar{\theta}(D')}_2 \right)
\\&\leq \exp \left(\eta \norm{\bar{\theta}(D') -\bar{\theta}(D)}_2\right)
\\&\leq \exp \left(\eta \cdot \Delta_i \bar{\theta} \right).
\end{align*}
where the second-to-last step comes from the triangle inequality and the last step comes from the definition of $\ell_2$-sensitivity with respect to agent $i$. 
\end{proof}

We can now conclude the proof. Picking $w_i = \frac{\varepsilon_i}{\sum_{j=1}^n \varepsilon_j}$ and $\eta = \frac{\lambda}{2 \sqrt{d}(\sqrt{d}B(\lambda) + 1)}\sum_{j=1}^n \varepsilon_j$, we get the result. Indeed:
\begin{itemize}
\item The weights are positive and immediately satisfy, $\sum_i w_i = 1$, as required per our algorithm.
\item The level of privacy obtained by agent $i$ is 
\[
\frac{2w_i \eta}{\lambda} \sqrt{d} (\sqrt{d} B(\lambda)+1) = \frac{2}{\lambda} \frac{\varepsilon_i}{\sum_{j=1}^n \varepsilon_j} \cdot \frac{\lambda}{2 \sqrt{d} (\sqrt{d} B(\lambda) + 1)} \left(\sum_{j=1}^n \varepsilon_j\right) \cdot \sqrt{d} (\sqrt{d} B(\lambda)+1) = \varepsilon_i.
\]
\end{itemize}

\subsection{Proof of Theorem~\ref{thm:accuracy}}\label{app:proof_accuracy}
We start by deriving a closed-form expression for $\bar{\theta}$. Note that for Algorithm~\ref{Alg:output_boundedlabel}, $\bar{\theta}$ is simply the unique solution to the unconstrained minimization problem 
\[
\argmin_{\theta \in \R^d}\sum_{i=1}^n w_i \left(y_i - \theta^\top x_i \right)^2 + \lambda \| \theta \|^2_2.
\]
Therefore, 
\[
\bar{\theta} = \argmin_{\theta \in \R^d}\sum_{i=1}^n w_i \left(y_i - \theta^\top x_i \right)^2 + \lambda \| \theta \|^2_2.
\]
Taking the first order-condition, we note that we must have 
\[
 -2 \sum_{i=1}^n w_i x_i \left(y_i - x_i^\top \bar{\theta}\right) + 2 \lambda \theta = 0.
\]
This can be rewritten as 
\[
2 \left(\sum_{i=1}^n w_i x_i x_i^\top + \lambda I \right) \bar{\theta} = 2 \sum_{i=1}^n w_i x_i y_i.
\]
Because $\lambda > 0$ and $\sum_{i=1}^n w_i x_i x_i^\top$ is positive semi-definite, $\sum_{i=1}^n w_i x_i x_i^\top + \lambda I$ is invertible, and we obtain the following closed-form expression for $\bar{\theta}$:
\[
\bar{\theta} = \left(\sum_{i=1}^n w_i x_i x_i^\top + \lambda I \right)^{-1} \sum_{i=1}^n w_i x_i y_i.
\]
We now rewrite this closed-form expression as a function of $\theta^*$, leveraging the fact that $y_i = x_i^\top \theta^* + Z_i$ for $Z_i \sim N(0,\sigma^2)$ for all observations $i \in [n]$. This yields:
\begin{align*}
\bar{\theta} 
&= \left(\lambda I + \sum_{i=1}^n w_i x_i x_i^\top \right)^{-1} \sum_{i=1}^n w_i x_i y_i\\
\\&= \left(\lambda I + \sum_{i=1}^n w_i x_i x_i^\top \right)^{-1} \sum_{i=1}^n w_i x_i (x_i^\top \theta^* + Z_i)
\\& = \left(\lambda I + \sum_{i=1}^n w_i x_i x_i^\top \right)^{-1} \left(\sum_{i=1}^n w_i x_i x_i^\top \theta^* + \sum_{i=1}^n w_i x_i Z_i\right)
\\&= \left(\lambda I + \sum_{i=1}^n w_i x_i x_i^\top \right)^{-1} \left( \left(\lambda I + \sum_{i=1}^n w_i x_i x_i^\top\right) \theta^* + \sum_{i=1}^n w_i x_i Z_i - \lambda \theta^*\right)
\\& = \theta^* + \left(\lambda I + \sum_{i=1}^n w_i x_i x_i^\top \right)^{-1} \left(\sum_{i=1}^n w_i x_i Z_i - \lambda \theta^*\right).
\end{align*}

In turn, we have that the distance between $\theta^*$ and $\hat{\theta}$ satisfies, where $Z$ is the random variable added for privacy as per Algorithms~\ref{Alg:output_boundedlabel}:
\begin{align*}
\left\Vert \hat{\theta} - \theta^* \right\Vert
&= \left\Vert \left(\lambda I + \sum_{i=1}^n w_i x_i x_i^\top \right)^{-1} \left(\sum_{i=1}^n w_i x_i Z_i - \lambda \theta^*\right) + Z \right \Vert
\\&\leq \left\Vert \left(\lambda I + \sum_{i=1}^n w_i x_i x_i^\top \right)^{-1} \left(\sum_{i=1}^n w_i x_i Z_i - \lambda \theta^*\right) \right\Vert + \left\Vert Z \right \Vert
\\&\leq \frac{1}{\lambda + \lambda_{min} \left(\sum_{i=1}^n w_i x_i x_i^\top\right)} \left\Vert \left(\sum_{i=1}^n w_i x_i Z_i - \lambda \theta^*\right) \right\Vert + \left\Vert Z \right \Vert
\\&\leq \frac{1}{1 + \frac{\lambda_{min} \left(\sum_{i=1}^n w_i x_i x_i^\top\right)}{\lambda}} \cdot \norm{\theta^*}
+ \frac{1}{\lambda} \left\Vert \sum_{i=1}^n w_i x_i Z_i \right\Vert + \left\Vert Z \right \Vert. 
\end{align*}

To conclude the proof, we write concentration inequalities on $\left\Vert \sum_{i=1}^n w_i x_i Z_i \right\Vert$ and on $\left\Vert Z \right \Vert$: 
\begin{itemize} 
\item For $Z$: we know per Remark~\ref{rmk:sampling} that $\Vert Z \Vert_2$ follows a gamma distribution with parameters $d$ and $\eta$; therefore, it has mean $\frac{d}{\eta}$ and variance $\frac{d}{\eta^2}$. By Chebyshev, we obtain that with probability at most $\delta/2$, $\Vert Z \Vert \geq \frac{d}{\eta} + \sqrt{\frac{2}{\delta}} \cdot \frac{\sqrt{d}}{\eta}$
\item For $\left\Vert \sum_{i=1}^n w_i x_i Z_i \right\Vert$: first note that $S = \sum_{i=1}^n w_i x_i Z_i$ is a multivariate Gaussian random variable. It has mean $0$ since the $Z_i$'s have mean $0$, and covariance 
\[
\Sigma = \sum_{i=1}^n \sum_j w_i w_j Cov(Z_i, Z_j) x_i x_j^\top = \sigma^2 \sum_{i=1}^n w_i^2 x_i x_i^\top, 
\]
where the equality comes from $Cov(Z_i,Z_i) = \sigma^2$ and $Cov(Z_i,Z_j) = 0$ for $i \neq j$ by independence. Note that then since $S(k)$ has mean $0$ for all $k \in [d]$,
\[
E\left[\Vert S \Vert^2\right] = \sum_{k=1}^d E \left[S(k)^2 \right] = \sum_{k=1}^d Cov(S(k),S(k)) = \sum_{k=1}^d \Sigma_{kk} = Tr(\Sigma). 
\]

Now, by Markov's inequality, we have that $P\left[\Vert S \Vert^2 \geq \frac{2Tr(\Sigma)}{\delta}\right] \leq \delta/2$, or equivalently $P\left[\Vert S \Vert \geq \sqrt{\frac{2Tr(\Sigma)}{\delta}}\right] \leq \delta/2$. To conclude the proof, we simply need to compute the trace of covariance matrix $\Sigma$. We have that
\[
Tr(\Sigma) = \sigma^2 \sum_{i=1}^n w_i^2 Tr(x_i x_i^\top ) = \sigma^2 \sum_{i=1}^n w_i^2 \sum_{k=1}^d x_{ik}^2 = \sigma^2 \sum_{i=1}^n w_i^2 \Vert x_i \Vert^2 \leq d \sigma^2 \sum_{i=1}^n w_i^2,
\]
using that $\Vert x \Vert \leq \sqrt{d}$. This directly implies that 
\[
P\left[\Vert S \Vert \geq \sqrt{\frac{2 d \sigma^2 \Vert \vec{w} \Vert^2}{\delta}}\right] \leq P\left[\Vert S \Vert \geq \sqrt{\frac{2 Tr(\Sigma)}{\delta}}\right] \leq \delta/2.
\]
Therefore, we have that $\frac{1}{\lambda}\Vert \sum_{i=1}^n w_i x_i Z_i \Vert \geq \frac{\sigma }{\lambda}\sqrt{\frac{2d}{\delta}} \Vert \vec{w} \Vert$ with probability at most $\delta/2$.

\end{itemize}
The result follows by union bound over the randomness of both $Z$ and $\sum_{i=1}^n w_i x_i Z_i$. 

\section{Additional Experiments}\label{app:experiments}
\subsection{Comparison to non-personalized differential privacy}\label{app:experiments_nonpers}
In Tables~\ref{tab:impact_epsc_syn_4_1_plevel_34_43_23} and~\ref{tab:impact_epsc_insurance_4_1_plevel_34_43_23} (respectively on our synthetic and real dataset) where we change the privacy level $\varepsilon_c$ for the conservative users. We observe that using personalized privacy still performs significantly better, but the performance improvements diminish as $\varepsilon_c$ increases. This is not surprising: as $\varepsilon_c$ increases, there is less and less variability across users' privacy levels, leading to less of a need for personalized privacy. Non-personalized privacy estimators start working better as the amount of noise they must add, which scales as a function of $\varepsilon_c$, starts largely decreasing. 
\begin{table}[!h]
\centering
\resizebox{0.8\textwidth}{!}{%
\begin{tabular}{|c | cc|cc|}
\hline
\begin{tabular}[c]{@{}c@{}}Privacy level \\ of Conservatives\\ ($\varepsilon_c$)\end{tabular} &
  \begin{tabular}[c]{@{}c@{}}Unregularized\\ test loss \\ (\ours)\end{tabular} &
  \begin{tabular}[c]{@{}c@{}}Unregularized\\ test loss \\ (non-personalized)\end{tabular} &
  \begin{tabular}[c]{@{}c@{}}Regularized\\ test loss \\ (\ours)\end{tabular} &
  \begin{tabular}[c]{@{}c@{}}Regularized\\ test loss\\ (non-personalized)\end{tabular} \\
\hline
0.01 & $\mathbf{1.44 \cross 10^{-2}}$ & 3.38              & $\mathbf{1.15}$                  & $1.05 \cross 10^{3}$ \\
0.05 & $\mathbf{1.41 \cross 10^{-2}}$ & $2.09 \cross 10^{-1}$ & $\mathbf{1.05}$                  & $6.08 \cross 10^{1}$   \\
0.10 & $\mathbf{1.39 \cross 10^{-2}}$ & $4.76 \cross 10^{-2}$ & $\mathbf{1.00}$                  & $1.16 \cross 10^{1}$   \\
0.20 & $\mathbf{1.41 \cross 10^{-2}}$ & $3.40 \cross 10^{-2}$ & $\mathbf{1.03}$                  & 7.21    \\
0.30 & $\mathbf{1.34 \cross 10^{-2}}$ & $2.08 \cross 10^{-2}$ & $\mathbf{8.74 \cross 10^{-1}}$ & 1.20    \\
0.40 & $\mathbf{1.32 \cross 10^{-2}}$ & $1.65 \cross 10^{-2}$ & $\mathbf{8.29 \cross 10^{-1}}$ & 1.84    \\
0.50 & $\mathbf{1.29 \cross 10^{-2}}$ & $1.45 \cross 10^{-2}$ & $\mathbf{7.26 \cross 10^{-1}}$ & 1.20  \\
\hline
\end{tabular}
}
\caption{Lower loss compared to standard DP on the synthetic dataset, while varying $\varepsilon_c$ (privacy level of the conservative users), keeping $f_c=0.54, f_m=0.37, \varepsilon_m = 0.5$ (same parameters as shown in~\cite{jorgensen}) and $\lambda = 100$.}
\label{tab:impact_epsc_syn_4_1_plevel_34_43_23}
\end{table}

\begin{table}[!h]
\centering
\resizebox{0.8\textwidth}{!}{%
\begin{tabular}{|c|cc|cc|}
\hline
\begin{tabular}[c]{@{}c@{}}Privacy level\\ of Conservatives\\ ($\varepsilon_c$)\end{tabular} &
  \begin{tabular}[c]{@{}c@{}}Unregularized\\ test loss\\ (PDP-OP)\end{tabular} &
  \begin{tabular}[c]{@{}c@{}}Unregularized\\ test loss\\ (non-personalized)\end{tabular} &
  \begin{tabular}[c]{@{}c@{}}Regularized\\ test loss\\ (PDP-OP)\end{tabular} &
  \begin{tabular}[c]{@{}c@{}}Regularized\\ test loss\\ (non-personalized)\end{tabular} \\
\hline
0.01 & $\mathbf{2.27 \cross 10^{-1}}$ & $2.76 \cross 10^{2}$  & $\mathbf{7.43 \cross 10^{-1}}$ & $9.87 \cross 10^{2}$  \\
0.05 & $\mathbf{2.28 \cross 10^{-1}}$ & $2.00 \cross 10^{1}$  & $\mathbf{7.32 \cross 10^{-1}}$ & $7.19 \cross 10^{1}$  \\
0.10 & $\mathbf{2.12 \cross 10^{-1}}$ & 5.05                  & $\mathbf{6.92 \cross 10^{-1}}$ & $1.80 \cross 10^{1}$  \\
0.20 & $\mathbf{1.97 \cross 10^{-1}}$ & 1.28                  & $\mathbf{6.42 \cross 10^{-1}}$ & 4.54                  \\
0.30 & $\mathbf{1.83 \cross 10^{-1}}$ & $5.84 \cross 10^{-1}$ & $\mathbf{5.75 \cross 10^{-1}}$ & 2.06                  \\
0.40 & $\mathbf{1.66 \cross 10^{-1}}$ & $3.48 \cross 10^{-1}$ & $\mathbf{5.26 \cross 10^{-1}}$ & 1.17                  \\
0.50 & $\mathbf{1.55 \cross 10^{-1}}$ & $2.33 \cross 10^{-1}$ & $\mathbf{4.90 \cross 10^{-1}}$ & $7.61 \cross 10^{-1}$ \\
\hline
\end{tabular}%
}
\caption{Lower loss compared to standard DP on the Medical cost dataset while varying $\varepsilon_c$ (privacy level of the conservative users), keeping $f_c=0.54, f_m=0.37, \varepsilon_m = 0.5$ (same parameters as shown in~\cite{jorgensen}) and $\lambda = 1$(analogous to Table~\ref{tab:impact_epsc_syn_4_1_plevel_34_43_23})}
\label{tab:impact_epsc_insurance_4_1_plevel_34_43_23}
\end{table}

In Tables~\ref{tab:impact_epsm_syn_4_1_plevel_34_43_23} and~\ref{tab:impact_epsm_insurance_4_1_plevel_34_43_23}, we change the privacy level $\varepsilon_m$ for the pragmatist (or medium privacy) users. We observe once again that using personalized privacy performs significantly better. The performance improvements are consistent across the board, noting that there is still a need for personalized privacy as we still have significant variability across user privacy preferences, with $34$ percent of users requiring a stringent privacy level of $0.01$ and $23$ percent a privacy level of $1$.

\begin{table}[!h]
\centering
\resizebox{0.8\textwidth}{!}{
\begin{tabular}{|c | cc|cc|}
\hline
\begin{tabular}[c]{@{}c@{}}Privacy level \\ of pragmatists \\ ($\varepsilon_m$)\end{tabular} &
  \begin{tabular}[c]{@{}c@{}}Unregularized\\ test loss \\ (\ours)\end{tabular} &
  \begin{tabular}[c]{@{}c@{}}Unregularized\\ test loss \\ (non-personalized)\end{tabular} &
  \begin{tabular}[c]{@{}c@{}}Regularized\\ test loss \\ (\ours)\end{tabular} &
  \begin{tabular}[c]{@{}c@{}}Regularized\\ test loss\\ (non-personalized)\end{tabular} \\
\hline
0.05 & $\mathbf{2.11 \cross 10^{-2}}$ & 9.68                 & $\mathbf{3.23}$ & $2.95 \cross 10^{3}$ \\
0.10 & $\mathbf{1.88 \cross 10^{-2}}$ & 5.76                 & $\mathbf{2.55}$ & $1.76 \cross 10^{3}$ \\
0.15 & $\mathbf{1.89 \cross 10^{-2}}$ & 8.46                 & $\mathbf{2.56}$ & $2.60 \cross 10^{3}$ \\
0.20 & $\mathbf{1.79 \cross 10^{-2}}$ & 4.84                 & $\mathbf{2.25}$ & $1.49 \cross 10^{3}$ \\
0.25 & $\mathbf{1.66 \cross 10^{-2}}$ & 3.17                 & $\mathbf{1.83}$ & $9.78 \cross 10^{2}$ \\
0.30 & $\mathbf{1.57 \cross 10^{-2}}$ & 1.48                 & $\mathbf{1.53}$ & $4.52 \cross 10^{2}$ \\
0.35 & $\mathbf{1.58 \cross 10^{-2}}$ & $6.45 \cross 10^{-1}$ & $\mathbf{1.61}$ & $1.97 \cross 10^{2}$ \\
0.40 & $\mathbf{1.48 \cross 10^{-2}}$ & 2.44                 & $\mathbf{1.33}$ & $7.56 \cross 10^{2}$ \\
0.45 & $\mathbf{1.47 \cross 10^{-2}}$ & 3.21                 & $\mathbf{1.25}$ & $9.86 \cross 10^{2}$ \\
0.50 & $\mathbf{1.44 \cross 10^{-2}}$ & 3.38                 & $\mathbf{1.15}$ & $1.05 \cross 10^{3}$ \\
\hline
\end{tabular}
}
\caption{Lower loss compared to standard DP on the synthetic dataset, while varying $\varepsilon_m$ (privacy level of the pragmatists), keeping $f_c=0.54, f_m=0.37, \varepsilon_c = 0.01$ (same parameters as shown in~\cite{jorgensen}) and $\lambda = 100$.}
\label{tab:impact_epsm_syn_4_1_plevel_34_43_23}
\end{table}

\begin{table}[!h]
\centering
\resizebox{0.8\textwidth}{!}{
\begin{tabular}{|c|cc|cc|}
\hline
\begin{tabular}[c]{@{}c@{}}Privacy level\\ of Conservatives\\ ($\varepsilon_m$)\end{tabular} &
  \begin{tabular}[c]{@{}c@{}}Unregularized\\ test loss\\ (PDP-OP)\end{tabular} &
  \begin{tabular}[c]{@{}c@{}}Unregularized\\ test loss\\ (non-personalized)\end{tabular} &
  \begin{tabular}[c]{@{}c@{}}Regularized\\ test loss\\ (PDP-OP)\end{tabular} &
  \begin{tabular}[c]{@{}c@{}}Regularized\\ test loss\\ (non-personalized)\end{tabular} \\
\hline
0.05 & $\mathbf{5.66 \cross 10^{-1}}$ & $4.99 \cross 10^{2}$ & $\mathbf{1.94}$                & $1.75 \cross 10^{3}$ \\
0.10 & $\mathbf{5.18 \cross 10^{-1}}$ & $4.99 \cross 10^{2}$ & $\mathbf{1.76}$                & $1.78 \cross 10^{3}$ \\
0.15 & $\mathbf{4.50 \cross 10^{-1}}$ & $5.02 \cross 10^{2}$ & $\mathbf{1.55}$                & $1.76 \cross 10^{3}$ \\
0.20 & $\mathbf{4.00\cross 10^{-1}}$  & $4.81 \cross 10^{2}$ & $\mathbf{1.38}$                & $1.71 \cross 10^{3}$ \\
0.25 & $\mathbf{3.60 \cross 10^{-1}}$ & $4.80 \cross 10^{2}$ & $\mathbf{1.22}$                & $1.70 \cross 10^{3}$ \\
0.30 & $\mathbf{3.27 \cross 10^{-1}}$ & $3.94 \cross 10^{2}$ & $\mathbf{1.10}$                & $1.43 \cross 10^{3}$ \\
0.35 & $\mathbf{2.85 \cross 10^{-1}}$ & $4.82 \cross 10^{2}$ & $\mathbf{9.75 \cross 10^{-1}}$ & $1.75 \cross 10^{3}$ \\
0.40 & $\mathbf{2.69 \cross 10^{-1}}$ & $4.83 \cross 10^{2}$ & $\mathbf{9.05 \cross 10^{-1}}$ & $1.71 \cross 10^{3}$ \\
0.45 & $\mathbf{2.53 \cross 10^{-1}}$ & $4.36 \cross 10^{2}$ & $\mathbf{8.28 \cross 10^{-1}}$ & $1.61 \cross 10^{3}$ \\
0.50 & $\mathbf{2.27 \cross 10^{-1}}$ & $2.76 \cross 10^{2}$ & $\mathbf{7.43 \cross 10^{-1}}$ & $9.87 \cross 10^{2}$ \\
\hline
\end{tabular}
}
\caption{Lower loss compared to standard DP on the Medical cost dataset, while varying $\varepsilon_m$ (privacy level of the pragmatists), keeping $f_c=0.54, f_m=0.37, \varepsilon_c = 0.01$ (same parameters as shown in~\cite{jorgensen}) and $\lambda = 1$. (analogous to Table~\ref{tab:impact_epsm_syn_4_1_plevel_34_43_23})}
\label{tab:impact_epsm_insurance_4_1_plevel_34_43_23}
\end{table}

In Tables~\ref{tab:impact_fc_syn_4_1_plevel_34_43_23} and \ref{tab:impact_fc_insurance_4_1_plevel_34_43_23}, we change the fraction of conservative users $f_c$. We note that even then, our personalized privacy framework still yields consistent and significant improvements over non-personalized privacy. This is because the existence of users with $\varepsilon_c = 0.01$ forces the non-personalized privacy estimate to still add noise that scales with this most stringent privacy requirement. 

\begin{table}[!h]
\centering
\resizebox{0.8\textwidth}{!}{%
\begin{tabular}{|c | cc|cc|}
\hline
\begin{tabular}[c]{@{}c@{}}Fraction of \\ Conservative \\ Users\\ ($f_c$)\end{tabular} &
  \begin{tabular}[c]{@{}c@{}}Unregularized\\ test loss \\ (\ours)\end{tabular} &
  \begin{tabular}[c]{@{}c@{}}Unregularized\\ test loss \\ (non-personalized)\end{tabular} &
  \begin{tabular}[c]{@{}c@{}}Regularized\\ test loss \\ (\ours)\end{tabular} &
  \begin{tabular}[c]{@{}c@{}}Regularized\\ test loss\\ (non-personalized)\end{tabular} \\
\hline
0.1 & $\mathbf{9.02 \cross 10^{-3}}$ & $7.87 \cross 10^{-1}$ & $\mathbf{4.91 \cross 10^{-1}}$ & $2.43 \cross 10^{2}$ \\
0.2 & $\mathbf{9.41 \cross 10^{-3}}$ & $3.80 \cross 10^{-1}$ & $\mathbf{6.51 \cross 10^{-1}}$ & $1.14 \cross 10^{2}$ \\
0.3 & $\mathbf{1.02 \cross 10^{-2}}$ & 4.56                  & $\mathbf{8.76 \cross 10^{-1}}$ & $1.41 \cross 10^{3}$ \\
0.4 & $\mathbf{1.14 \cross 10^{-2}}$ & 3.82                  & \textbf{1.27}                  & $1.17 \cross 10^{3}$ \\
0.5 & $\mathbf{1.30 \cross 10^{-2}}$ & 7.00                  & \textbf{1.75}                  & $2.14 \cross 10^{3}$ \\
0.6 & $\mathbf{1.69 \cross 10^{-2}}$ & 9.59                  & \textbf{2.92}                  & $2.90 \cross 10^{3}$ \\
\hline
\end{tabular}%
}
\caption{Lower loss compared to standard DP on the synthetic dataset, while varying $f_c$ (fraction of conservative users) for $f_m=0.37,~f_l = 1 - f_c - f_m,~\varepsilon_c=0.01,~\varepsilon_m=0.2$, $\varepsilon_l=1.0, \lambda = 100$}
\label{tab:impact_fc_syn_4_1_plevel_34_43_23}
\end{table}

\begin{table}[!h]
\centering
\resizebox{0.8\textwidth}{!}{%
\begin{tabular}{|c|cc|cc|}
\hline
\begin{tabular}[c]{@{}c@{}}Fraction of\\ Conservative\\ Users\\ ($f_c$)\end{tabular} &
  \begin{tabular}[c]{@{}c@{}}Unregularized\\ test loss\\ (PDP-OP)\end{tabular} &
  \begin{tabular}[c]{@{}c@{}}Unregularized\\ test loss\\ (non-personalized)\end{tabular} &
  \begin{tabular}[c]{@{}c@{}}Regularized\\ test loss\\ (PDP-OP)\end{tabular} &
  \begin{tabular}[c]{@{}c@{}}Regularized\\ test loss\\ (non-personalized)\end{tabular} \\
\hline
0.1 & $\mathbf{1.17 \cross 10^{-1}}$ & $4.95 \cross 10^{2}$ & $\mathbf{3.47 \cross 10^{-1}}$ & $1.77 \cross 10^{3}$ \\
0.2 & $\mathbf{1.45 \cross 10^{-1}}$ & $4.89 \cross 10^{2}$ & $\mathbf{4.46 \cross 10^{-1}}$ & $1.74 \cross 10^{3}$ \\
0.3 & $\mathbf{1.82 \cross 10^{-1}}$ & $5.00 \cross 10^{2}$ & $\mathbf{5.78 \cross 10^{-1}}$ & $1.79 \cross 10^{3}$ \\
0.4 & $\mathbf{2.35 \cross 10^{-1}}$ & $4.80 \cross 10^{2}$ & $\mathbf{7.71 \cross 10^{-1}}$ & $1.73 \cross 10^{3}$ \\
0.5 & $\mathbf{3.45 \cross 10^{-1}}$ & $4.93 \cross 10^{2}$ & $\mathbf{1.16}$                & $1.80 \cross 10^{3}$ \\
0.6 & $\mathbf{5.61 \cross 10^{-1}}$ & $4.19 \cross 10^{2}$ & $\mathbf{1.92}$                & $1.51 \cross 10^{3}$ \\
\hline
\end{tabular}%
}
\caption{Lower loss compared to standard DP on the Medical cost dataset, while varying $f_c$(fraction of conservative users) $f_m = 0.37, f_l = 1 \minus f_c \minus f_m, \varepsilon_c = 0.01, \varepsilon_m = 0.2$, $\varepsilon_l = 1.0, \lambda = 1$.}
\label{tab:impact_fc_insurance_4_1_plevel_34_43_23}
\end{table}

Finally, in Tables~\ref{tab:impact_n_syn_4_1_plevel_34_43_23} and \ref{tab:impact_n_insurance_4_1_plevel_34_43_23}, we fix a single value of the regularization parameter $\lambda$ and of the distribution of privacy requirements of the users and show how our results evolve as the number of samples we feed our algorithm increases. Unsurprisingly, the more samples we have access to, the better the performance of both our personalized privacy approach as well as the non-personalized baseline. Our approach continues to see significant and consistent improvements compared to the non-personalized baseline.

\begin{table}[!h]
\centering
\resizebox{0.8\textwidth}{!}{
\begin{tabular}{|c|cc|cc|}
\hline
\begin{tabular}[c]{@{}c@{}}Fraction of\\ training samples\\ ($n$)\end{tabular} &
  \begin{tabular}[c]{@{}c@{}}Unregularized\\ test loss\\ (PDP-OP)\end{tabular} &
  \begin{tabular}[c]{@{}c@{}}Unregularized\\ test loss\\ (non-personalized)\end{tabular} &
  \begin{tabular}[c]{@{}c@{}}Regularized\\ test loss\\ (PDP-OP)\end{tabular} &
  \begin{tabular}[c]{@{}c@{}}Regularized\\ test loss\\ (non-personalized)\end{tabular} \\
\hline
0.1 & $\mathbf{4.28 \cross 10^{-1}}$ & $1.22 \cross 10^{1}$ & $\mathbf{1.32 \cross 10^{2}}$ & $3.73 \cross 10^{3}$ \\
0.2 & $\mathbf{1.41 \cross 10^{-1}}$ & $1.15 \cross 10^{2}$ & $\mathbf{4.02 \cross 10^{1}}$ & $3.45 \cross 10^{4}$ \\
0.3 & $\mathbf{5.01 \cross 10^{-2}}$ & 8.97                 & $\mathbf{1.46 \cross 10^{1}}$ & $2.89 \cross 10^{3}$ \\
0.4 & $\mathbf{4.64 \cross 10^{-2}}$ & 2.37                 & $\mathbf{1.27 \cross 10^{1}}$ & $7.24 \cross 10^{2}$ \\
0.5 & $\mathbf{3.49 \cross 10^{-2}}$ & $2.14 \cross 10^{1}$ & $\mathbf{8.11}$               & $6.61 \cross 10^{3}$ \\
0.6 & $\mathbf{2.12 \cross 10^{-2}}$ & $1.59 \cross 10^{1}$ & $\mathbf{5.16}$               & $5.06 \cross 10^{3}$ \\
0.7 & $\mathbf{2.00 \cross 10^{-2}}$ & 7.97                 & $\mathbf{4.66}$               & $2.52 \cross 10^{3}$ \\
0.8 & $\mathbf{1.64 \cross 10^{-2}}$ & $1.03 \cross 10^{1}$ & $\mathbf{3.36}$               & $2.95 \cross 10^{3}$ \\
0.9 & $\mathbf{1.10 \cross 10^{-2}}$ & 6.36                 & $\mathbf{2.52}$               & $1.93 \cross 10^{3}$ \\
1.0 & $\mathbf{1.01 \cross 10^{-2}}$ & 1.98                 & $\mathbf{2.16}$               & $6.17 \cross 10^{2}$ \\
\hline
\end{tabular}
}
\caption{Lower loss compared to standard DP on the synthetic dataset, while varying the fraction of the training set samples we use. For example, here $n = 0.3$ means that we use a 0.3 fraction of the training set. We fix $f_c = 0.34, f_m = 0.43, f_l = 0.23, \varepsilon_c = 0.01, \varepsilon_m = 0.2, \varepsilon_l = 1.0$ and $\lambda = 100$.}
\label{tab:impact_n_syn_4_1_plevel_34_43_23}
\end{table}

\begin{table}[!h]
\centering
\resizebox{0.8\textwidth}{!}{%
\begin{tabular}{|c|cc|cc|}
\hline
\begin{tabular}[c]{@{}c@{}}Fraction of\\ training samples\\ ($n$)\end{tabular} &
  \begin{tabular}[c]{@{}c@{}}Unregularized\\ test loss\\ (PDP-OP)\end{tabular} &
  \begin{tabular}[c]{@{}c@{}}Unregularized\\ test loss\\ (non-personalized)\end{tabular} &
  \begin{tabular}[c]{@{}c@{}}Regularized\\ test loss\\ (PDP-OP)\end{tabular} &
  \begin{tabular}[c]{@{}c@{}}Regularized\\ test loss\\ (non-personalized)\end{tabular} \\
\hline
0.1 & $\mathbf{1.83 \cross 10^1}$    & $1.45 \cross 10^4$ & $\mathbf{6.64 \cross 10^{1}}$  & $5.21 \cross 10^4$ \\
0.2 & $\mathbf{4.63}$                & $6.62 \cross 10^3$ & $\mathbf{1.65 \cross 10^{1}}$  & $2.37 \cross 10^4$ \\
0.3 & $\mathbf{2.05}$                & $5.23 \cross 10^3$ & $\mathbf{7.29}$                & $1.85 \cross 10^4$ \\
0.4 & $\mathbf{1.21}$                & $1.60 \cross 10^3$ & $\mathbf{4.28}$                & $5.77 \cross 10^3$ \\
0.5 & $\mathbf{7.85 \cross 10^{-1}}$ & $1.96 \cross 10^3$ & $\mathbf{2.75}$                & $6.98 \cross 10^3$ \\
0.6 & $\mathbf{5.40 \cross 10^{-1}}$ & $1.36 \cross 10^3$ & $\mathbf{1.84}$                & $4.91 \cross 10^3$ \\
0.7 & $\mathbf{4.02 \cross 10^{-1}}$ & $9.64 \cross 10^2$ & $\mathbf{1.38}$                & $3.47 \cross 10^3$ \\
0.8 & $\mathbf{3.16 \cross 10^{-1}}$ & $7.26 \cross 10^2$ & $\mathbf{1.06}$                & $2.60 \cross 10^3$ \\
0.9 & $\mathbf{2.53 \cross 10^{-1}}$ & $6.15 \cross 10^2$ & $\mathbf{8.56 \cross 10^{-1}}$ & $2.24 \cross 10^3$ \\
1.0 & $\mathbf{2.14 \cross 10^{-1}}$ & $3.82 \cross 10^2$ & $\mathbf{6.99 \cross 10^{-1}}$ & $1.36 \cross 10^3$ \\
\hline
\end{tabular}
}
\caption{Lower loss compared to standard DP on the Medical cost dataset, while varying the number of samples we use. For example, here $n=0.3$ means that we use a $0.3$ fraction of the training set. We fix $f_c = 0.34, f_m = 0.43, f_l = 0.23, \varepsilon_c = 0.01, \varepsilon_m = 0.2, \varepsilon_l = 1.0$ and $\lambda = 1$.}
\label{tab:impact_n_insurance_4_1_plevel_34_43_23}
\end{table}

\subsection{Comparison to~\citet{jorgensen}}\label{app:experiments_jorgensen}
We provide additional experimental results that further the comparison of our \ours~algorithm with that of~\citet{jorgensen}. 

\paragraph{Improvements in loss.} In Figure~\ref{fig:syn_jorgensen_epsc} and Figure~\ref{fig:insurance_jorgensen_epsc} we vary the privacy level $\varepsilon_c$ for the conservative users, and observe that our algorithm always leads to a lower loss compared to \citet{jorgensen}. We further note that the performance improvements diminish as $\varepsilon_c$ increases, as there is less variability in the users' privacy levels and we get closer to the non-personalized case. 

\begin{figure}[!h]
\centering
\subfloat[Unregularized Loss ($\lambda = 50$)]{
    \includegraphics[width=0.4\textwidth]{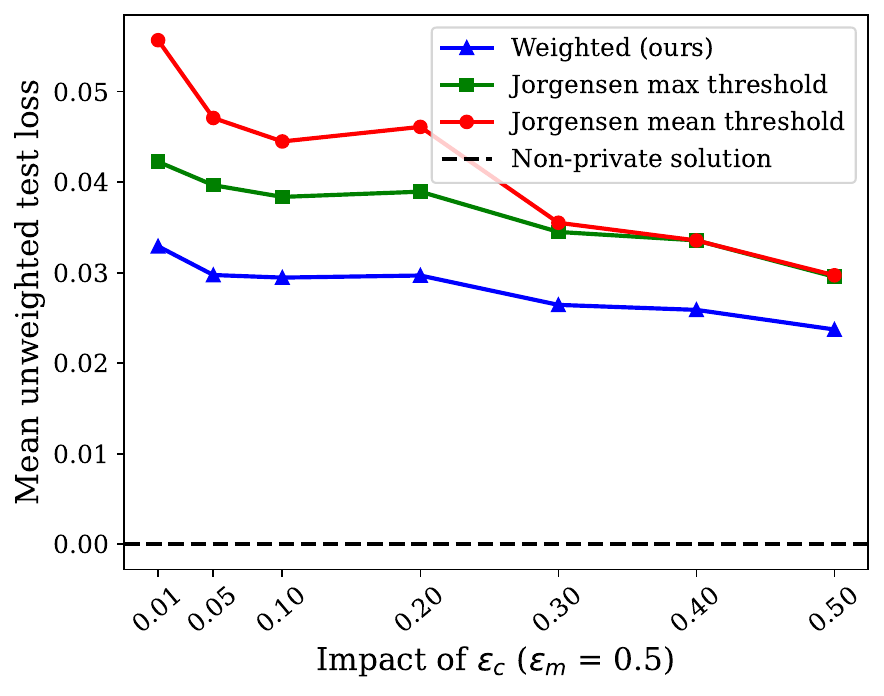}
} 
\subfloat[Regularized Loss ($\lambda = 50$)]{
    \includegraphics[width=0.4\textwidth]{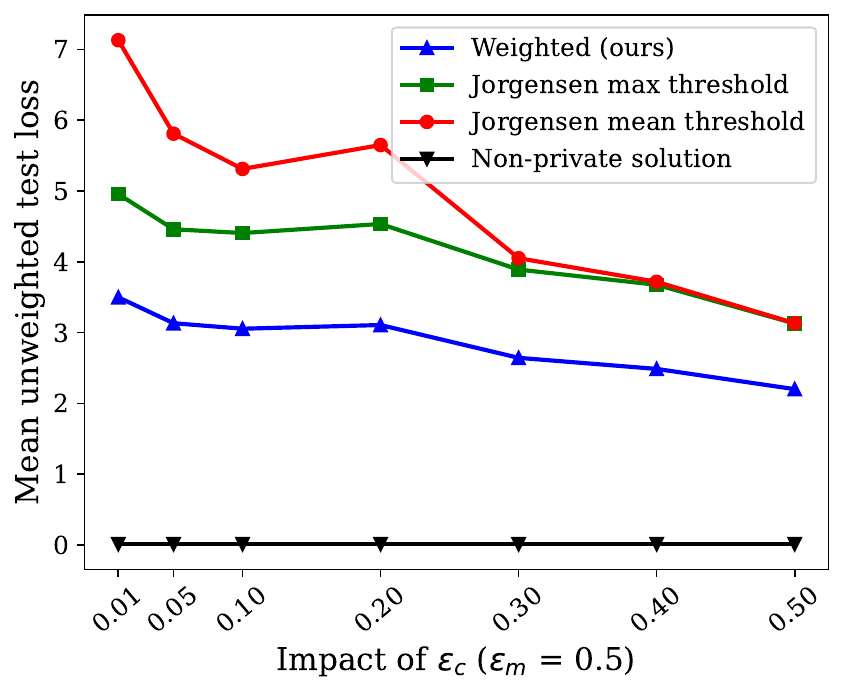}
}
\caption{Lower loss compared to \citet{jorgensen} on the synthetic dataset while varying $\varepsilon_c$ (privacy level of the conservative users) , keeping $\varepsilon_m = 0.5,~\varepsilon_l=1.0,~f_c = 0.54,~f_m = 0.37, f_l = 0.09$. \label{fig:syn_jorgensen_epsc}} 
\end{figure}

\begin{figure}[!h]
\centering
\subfloat[Unregularized Loss ($\lambda = 1$)]{
    \includegraphics[width=0.4\textwidth]{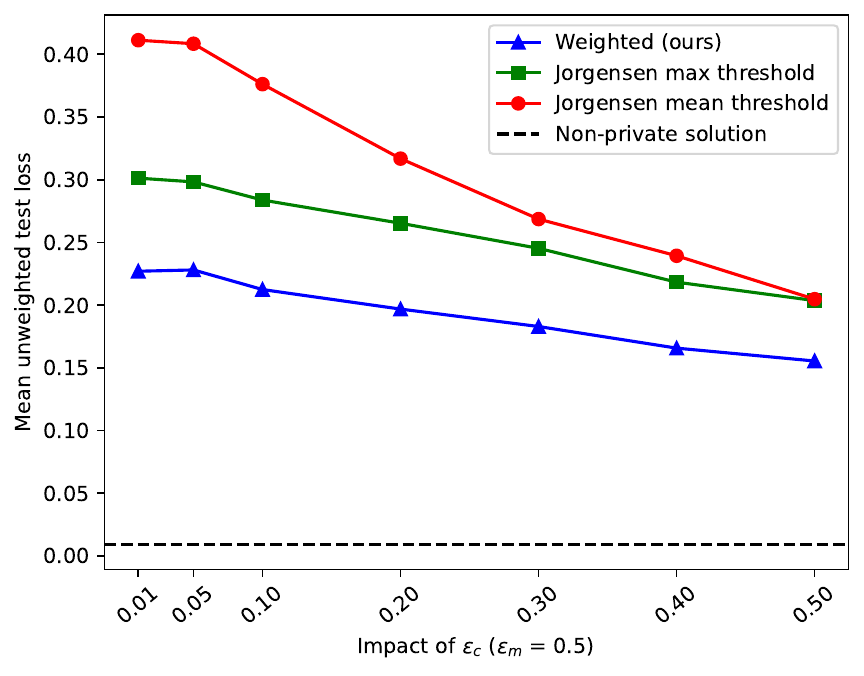}
} 
\subfloat[Regularized Loss ($\lambda = 1$)]{
    \includegraphics[width=0.4\textwidth]{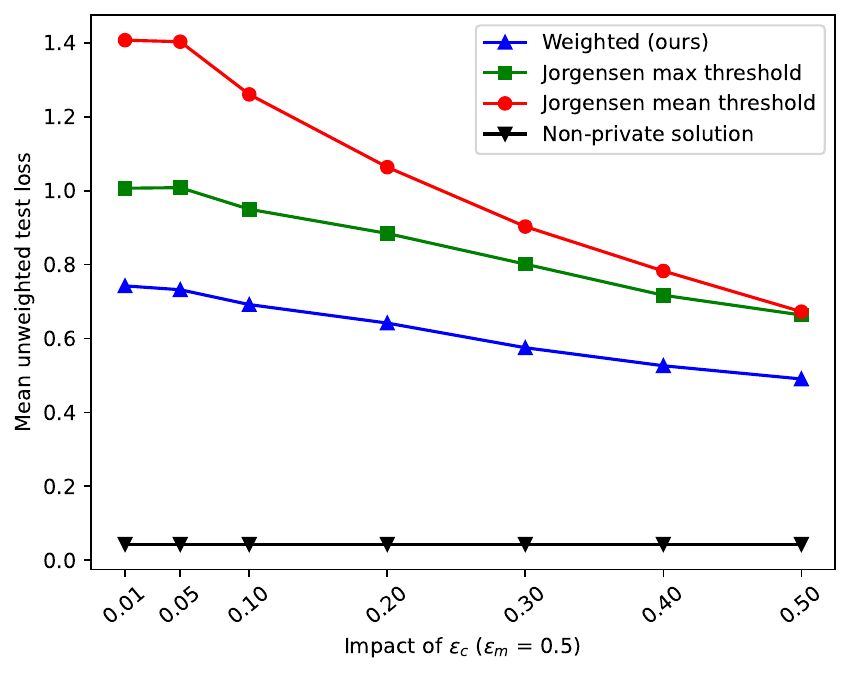}
}
\caption{Lower loss compared to \citet{jorgensen} on the Medical costs dataset while varying $\varepsilon_c$ (privacy level of the conservative users), keeping $\varepsilon_m = 0.5, \varepsilon_l=1.0, f_c = 0.54, f_m = 0.37, f_l = 0.09$. \label{fig:insurance_jorgensen_epsc}}
\end{figure}

Figures~\ref{fig:syn_jorgensen_epsm} and~\ref{fig:insurance_jorgensen_epsm} show similar insights when varying the parameter $\varepsilon_m$, that controls the privacy level of medium-privacy (or pragmatic) users.

\begin{figure*}[h]
\centering
\subfloat[Unregularized Loss ($\lambda = 50$)]{
\includegraphics[width=0.4\textwidth]{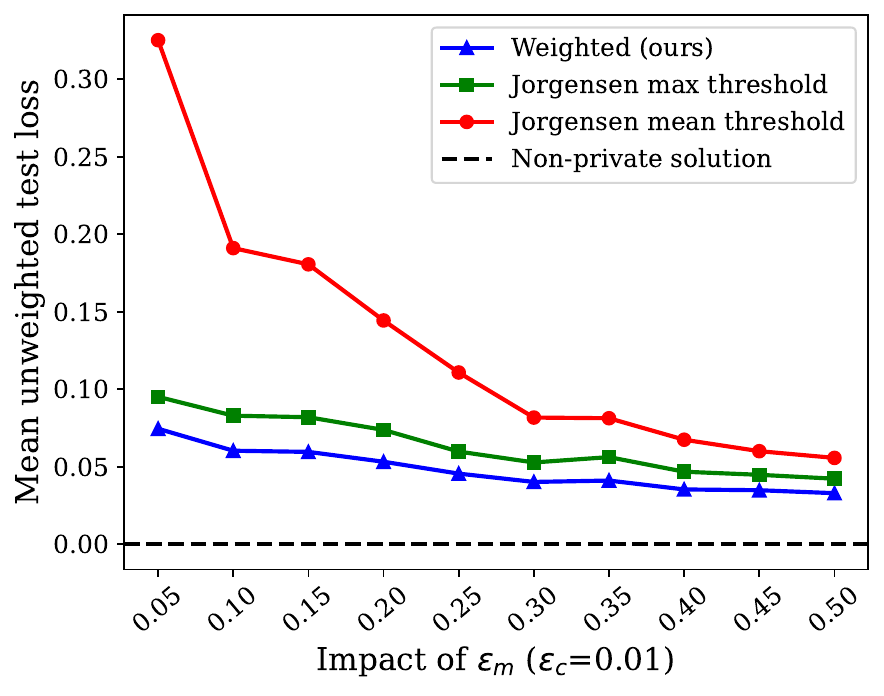}
}
\subfloat[Regularized Loss ($\lambda = 50$)]{
\includegraphics[width=0.4\textwidth]{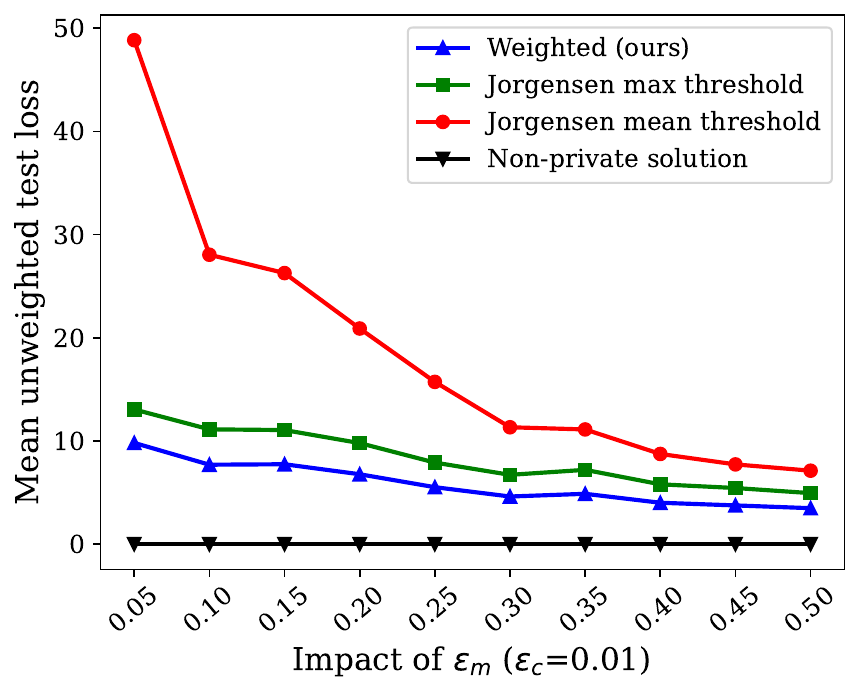}
}
\caption{Lower loss compared to \citet{jorgensen} on the synthetic dataset while varying $\varepsilon_m$ (privacy level of pragmatists), keeping $\varepsilon_c = 0.01, \varepsilon_l=1.0, f_c = 0.54, f_m = 0.37, f_l = 0.09$.}\label{fig:syn_jorgensen_epsm}
\end{figure*}

\begin{figure*}[h]
\centering
\subfloat[Unregularized Loss ($\lambda = 1$)]{
\includegraphics[width=0.4\textwidth]{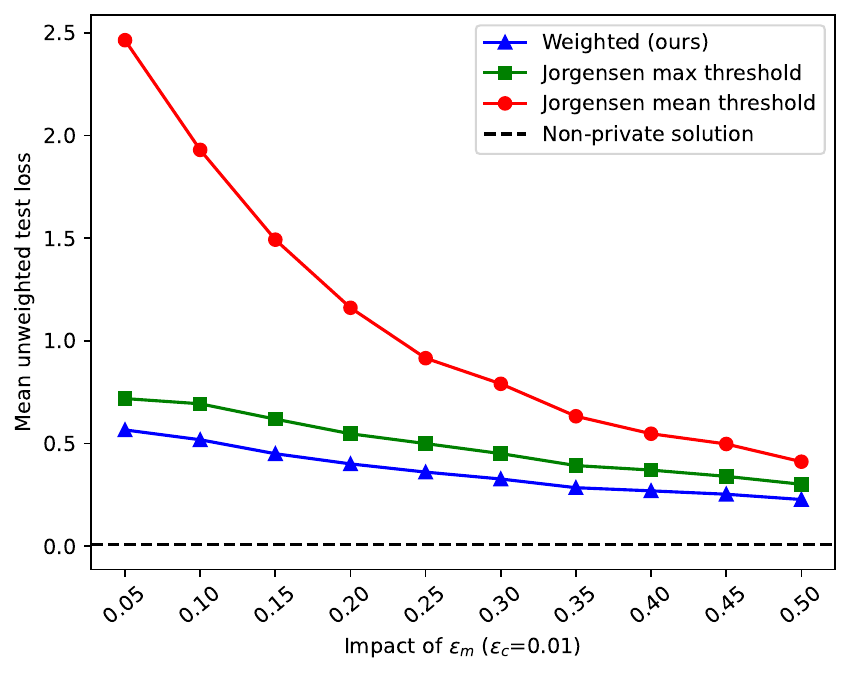}
}
\subfloat[Regularized Loss ($\lambda = 1$)]{
\includegraphics[width=0.4\textwidth]{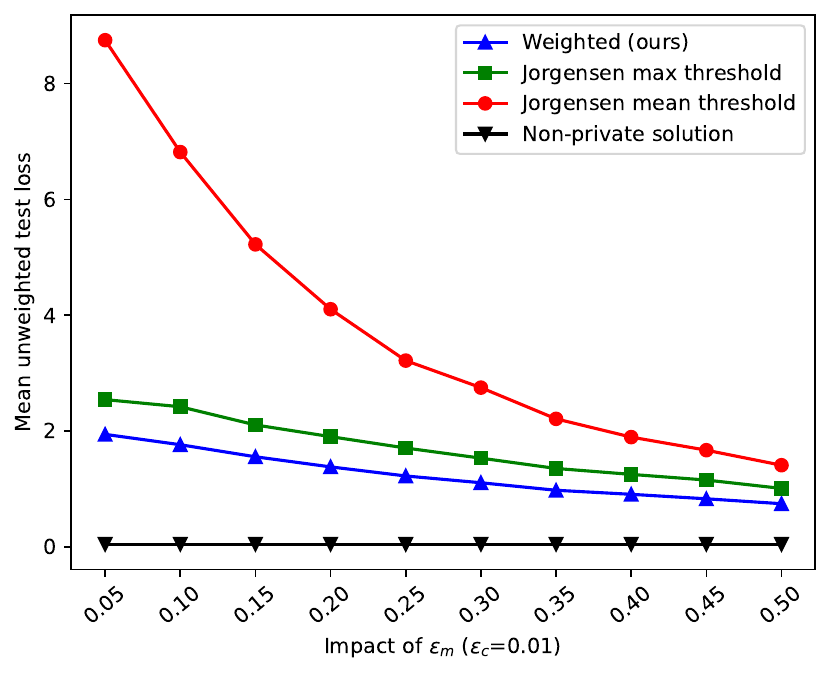}
}
\caption{Lower loss compared to \citet{jorgensen} for the Medical costs dataset while varying $\varepsilon_m$ (privacy level of pragmatists), keeping $\varepsilon_c = 0.01, \varepsilon_l=1.0, f_c = 0.54, f_m = 0.37, f_l = 0.09$.}\label{fig:insurance_jorgensen_epsm}
\end{figure*}

In Figures~\ref{fig:syn_jorgensen_fc} and~\ref{fig:insurance_jorgensen_fc} we vary the fraction of users with strong privacy requirements. We remark that our relative performance improvement compared to~\citet{jorgensen} becomes bigger as a larger fraction of users has high privacy requirements, highlighting the benefit of our framework in stringent privacy regimes. 

\begin{figure*}[!h]
\centering
\subfloat[Unregularized Loss ($\lambda = 100$)]{
\includegraphics[width=0.4\textwidth]{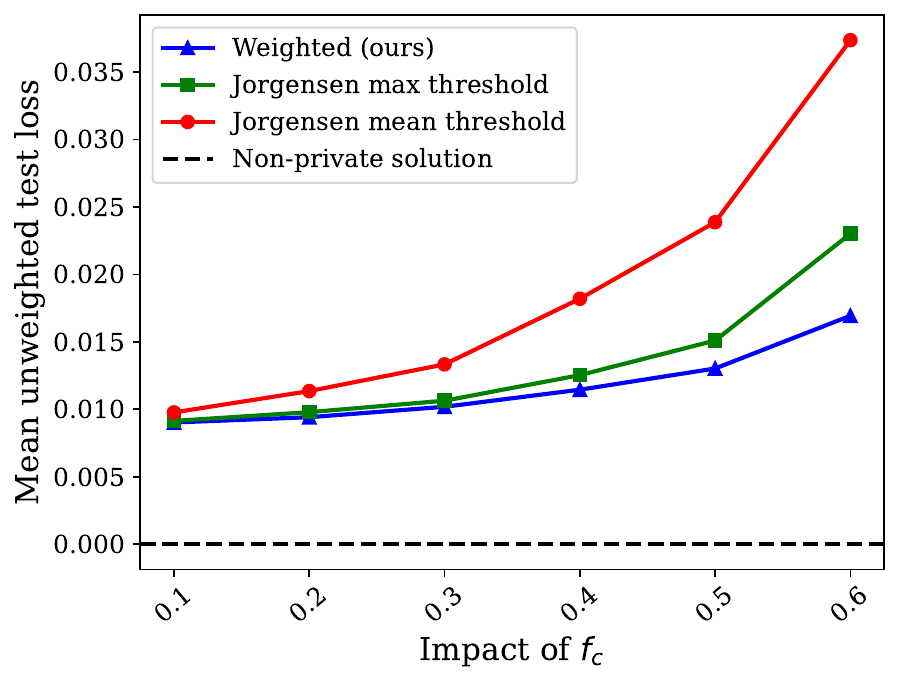}
}
\subfloat[Regularized Loss ($\lambda = 100$)]{
\includegraphics[width=0.4\textwidth]{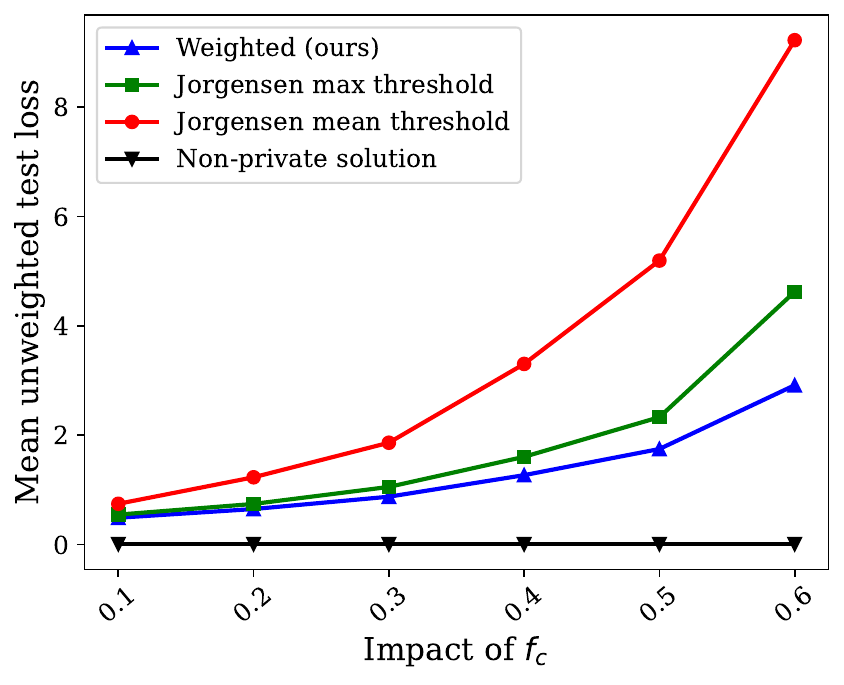}
}
\caption{Lower loss compared to \citet{jorgensen} on the synthetic dataset while varying $f_c$(fraction of conservative users), keeping $\varepsilon_c=0.01, \varepsilon_m = 0.2, \varepsilon_l=1.0, f_m = 0.37, f_l = 1 - f_m - f_c = 0.63 - f_c$.}\label{fig:syn_jorgensen_fc}
\end{figure*}

\begin{figure*}[!h]
\centering
\subfloat[Unregularized Loss ($\lambda = 0.5$)]{
\includegraphics[width=0.4\textwidth]{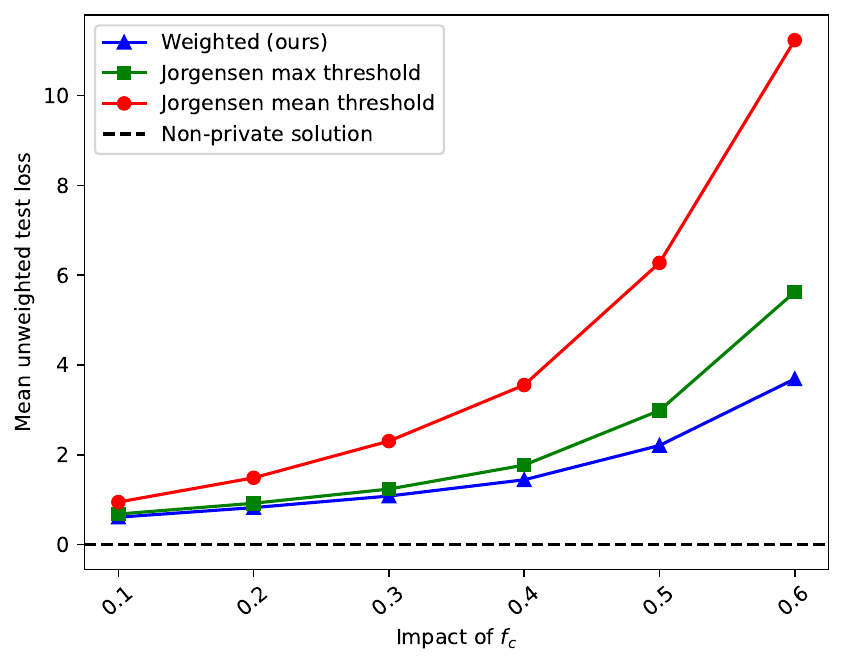}}
\subfloat[Regularized Loss ($\lambda = 0.5$)]{
\includegraphics[width=0.4\textwidth]{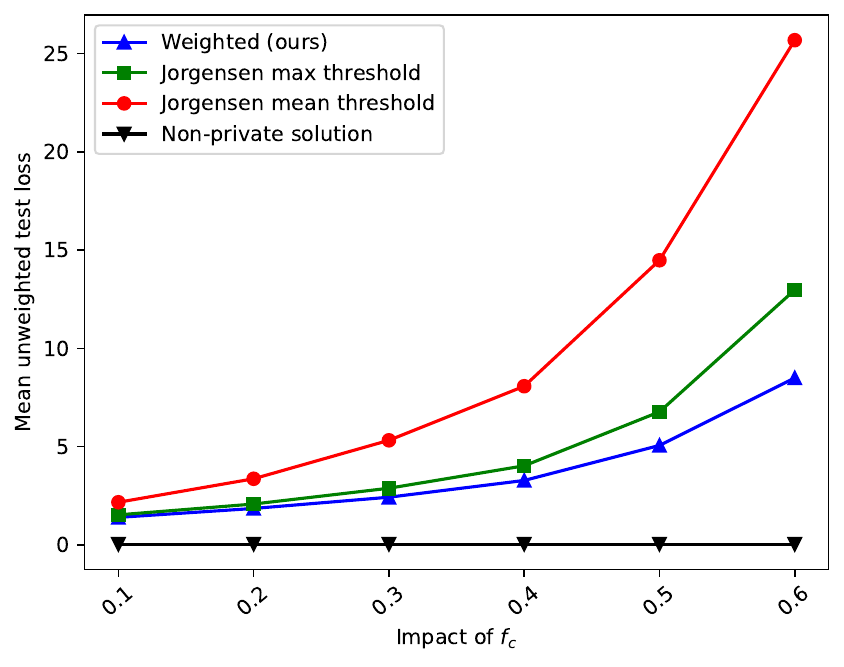}}
\caption{Lower loss compared to \citet{jorgensen} on the Medical costs dataset while $f_c$ (fraction of conservative users), keeping $\varepsilon_c=0.01, \varepsilon_m = 0.2, \varepsilon_l=1.0, f_m = 0.37, f_l = 1 - f_m - f_c = 0.63 - f_c$.}\label{fig:insurance_jorgensen_fc}
\end{figure*}

Finally, Figures~\ref{fig:syn_jorgensen_n} and~\ref{fig:insurance_jorgensen_n} vary the fraction of the training set used. We note that our loss improvements compared to~\citet{jorgensen} become less noticeable as $n$ increases. This is perhaps unsurprising, as the more data points we use, the lesser the impact of adding noise for privacy is, and there is much less leeway for improvement across different techniques for privacy. 

\paragraph{Improvements in variability.} We start with Figures~\ref{fig:syn_jorgensen_epsc_std} and ~\ref{fig:insurance_jorgensen_epsc_std} where we vary the privacy level $\varepsilon_c$ for the conservative users. Then, on Figures~\ref{fig:syn_jorgensen_epsm_std}   and~\ref{fig:insurance_jorgensen_epsm_std}, we focus on the case of varying $\varepsilon_m$. On Figures~\ref{fig:syn_jorgensen_fc_std} and~\ref{fig:insurance_jorgensen_fc_std}, we vary the fraction of users that have high privacy requirements. Finally, on Figures~\ref{fig:syn_jorgensen_n_std} and~\ref{fig:insurance_jorgensen_n_std}, we vary  the fraction of the training set used. On all figures, we note that the standard deviation of the loss of our PDP-OP algorithm is lower than \citet{jorgensen}.

\begin{figure*}[!h]
\centering
\subfloat[Unregularized Loss ($\lambda = 100$)]{
\includegraphics[width=0.4\textwidth]{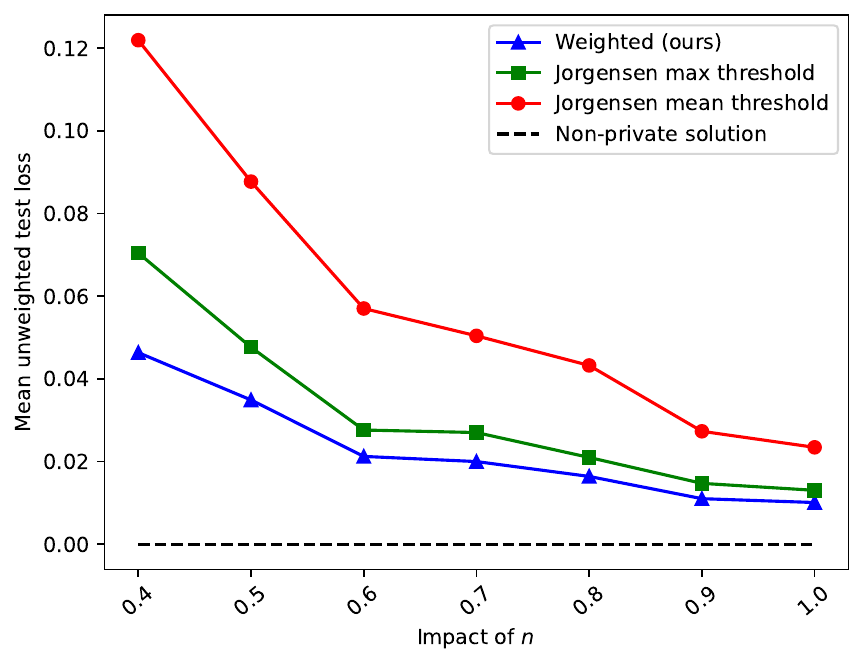}
}
\subfloat[Regularized Loss ($\lambda = 100)$)]{
\includegraphics[width=0.4\textwidth]{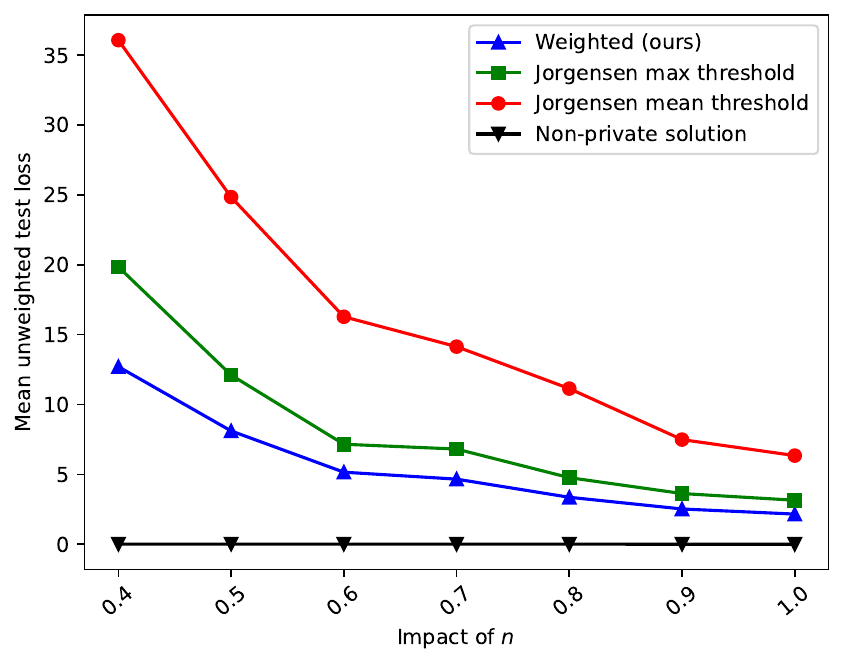}
}
    \caption{Lower loss compared to \citet{jorgensen} on the synthetic dataset while varying the parameter $n$ (the fraction of training samples used), keeping $\varepsilon_c=0.01, \varepsilon_m = 0.2, \varepsilon_c=1.0, f_c = 0.34, f_m = 0.43, f_l = 0.23$. }\label{fig:syn_jorgensen_n}
\end{figure*}

\begin{figure*}[!h]
\centering
\subfloat[Unregularized Loss ($\lambda = 1$)]{
\includegraphics[width=0.4\textwidth]{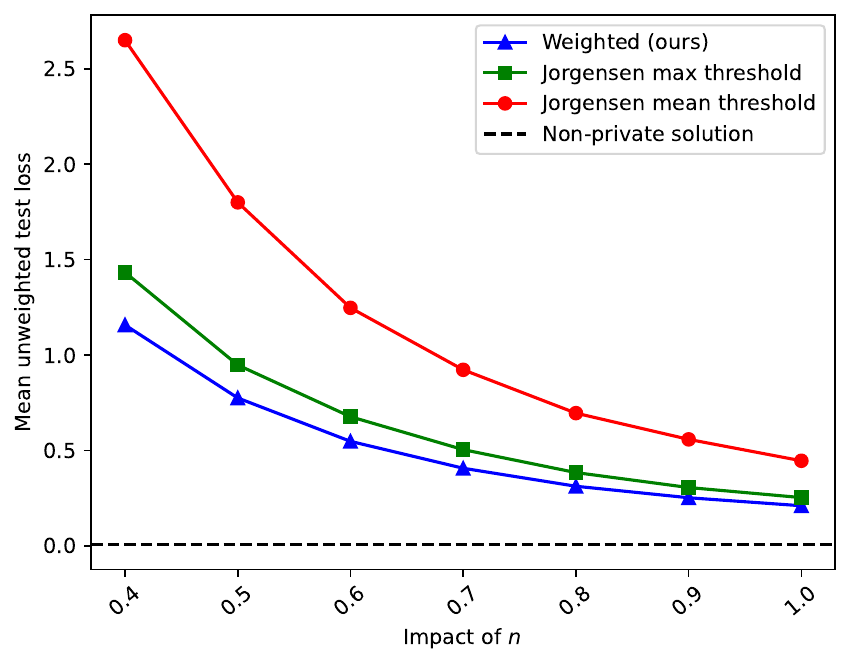}
}
\subfloat[Regularized Loss ($\lambda = 1)$)]{
\includegraphics[width=0.4\textwidth]{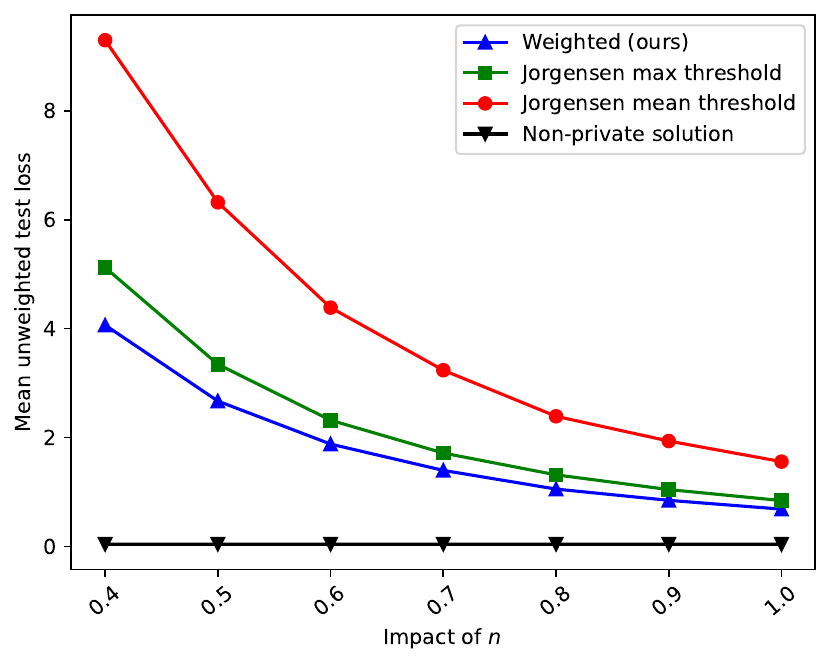}
}
    \caption{Lower loss compared to \citet{jorgensen} on the Medical costs dataset while varying $n$ (the fraction of training samples used), keeping $\varepsilon_c=0.01, \varepsilon_m = 0.2, \varepsilon_c=1.0, f_c = 0.34, f_m = 0.43$. }\label{fig:insurance_jorgensen_n}
\end{figure*}

\begin{figure*}[!h]
    \centering
    \subfloat[Unregularized Loss ($\lambda = 50$)] {
    \includegraphics[width=0.4\textwidth]{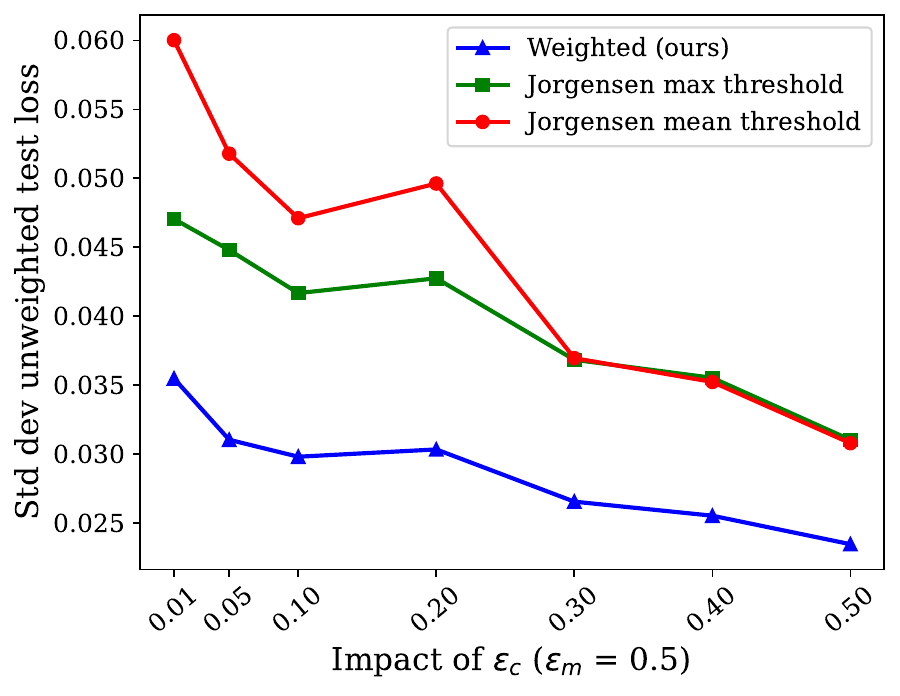}
    }
    \subfloat[Regularized Loss ($\lambda = 50$)]{\includegraphics[width=0.4\textwidth]{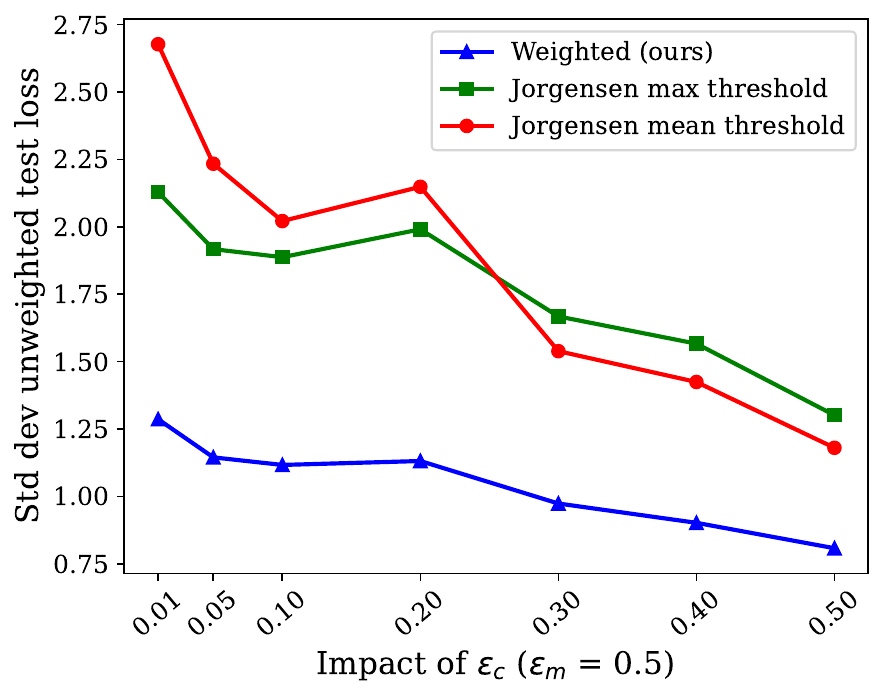}
    }
    \caption{Lower standard deviation compared to~\cite{jorgensen} on the synthetic dataset, while varying the $\varepsilon_c$ (the privacy level of conservative users), keeping $\varepsilon_m = 0.5, \varepsilon_l=1.0, f_c = 0.54, f_m = 0.37, f_l = 0.09$.}\label{fig:syn_jorgensen_epsc_std}
\end{figure*}

\begin{figure*}[!h]
    \centering
    \subfloat[Unregularized Loss ($\lambda = 1$)] {
    \includegraphics[width=0.4\textwidth]{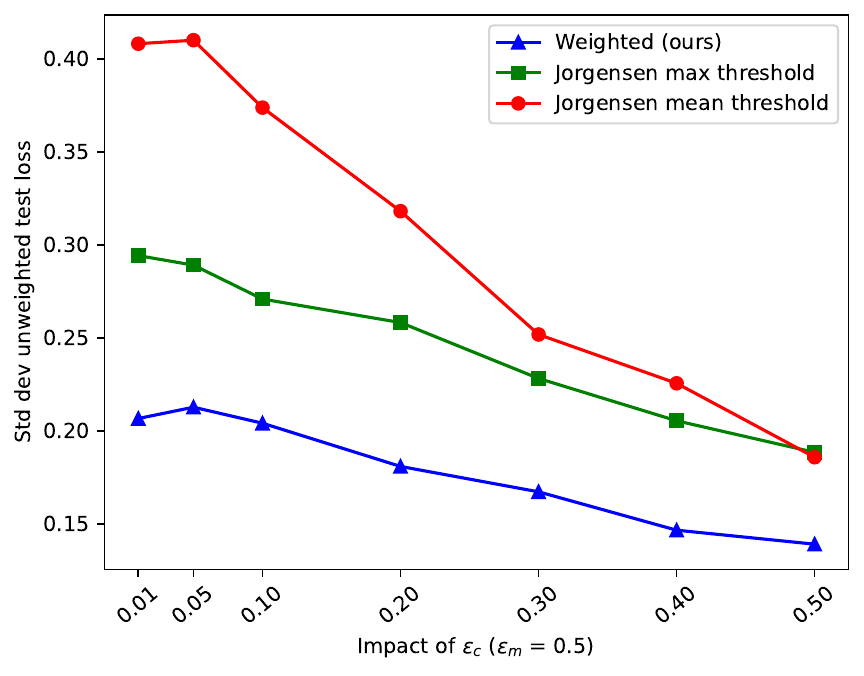}
    }
    \subfloat[Regularized Loss ($\lambda = 1$)]{\includegraphics[width=0.4\textwidth]{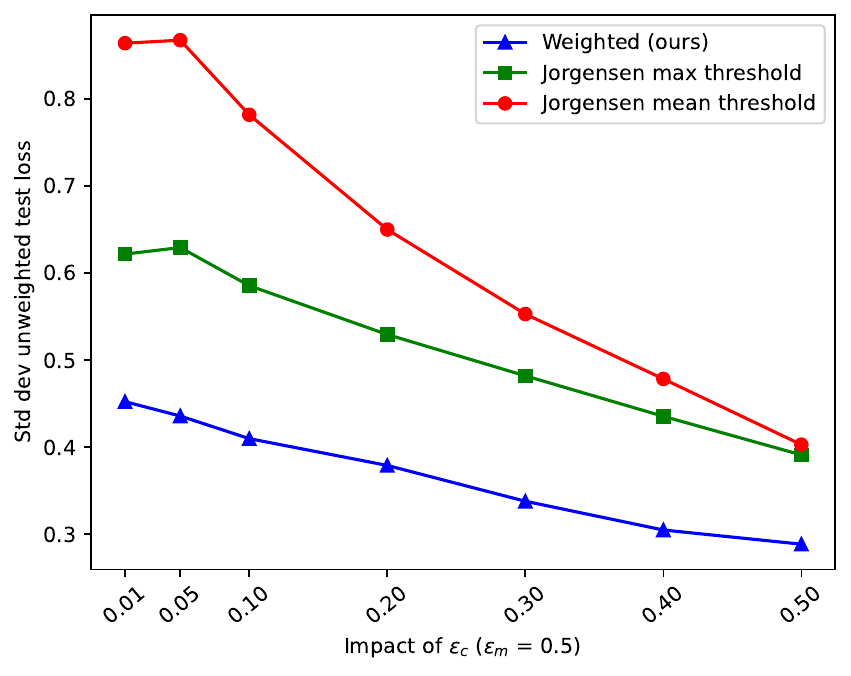}
    }
    \caption{Lower standard deviation compared to~\cite{jorgensen} on the Medical cost dataset, while varying the $\varepsilon_c$ (the privacy level of conservative users), keeping $\varepsilon_m = 0.5, \varepsilon_l=1.0, f_c = 0.54, f_m = 0.37, f_l = 0.09$. }\label{fig:insurance_jorgensen_epsc_std}
\end{figure*}

\begin{figure*}[!h]
    \centering
    \subfloat[Unregularized Loss ($\lambda = 50$)] {
        \includegraphics[width=0.4\textwidth]{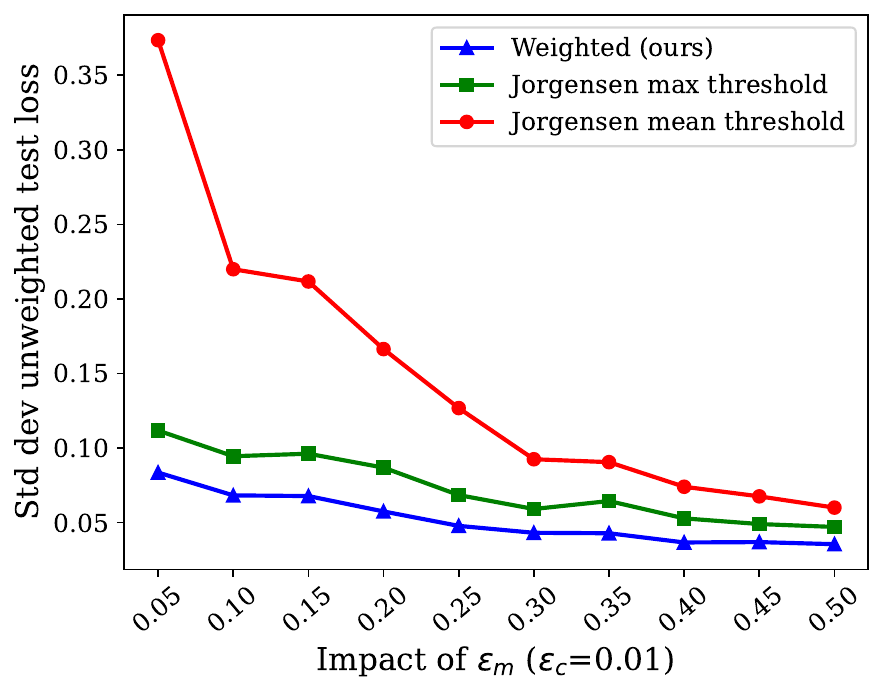}
    }
    \subfloat[Regularized Loss ($\lambda = 50$)]{
        \includegraphics[width=0.4\textwidth]{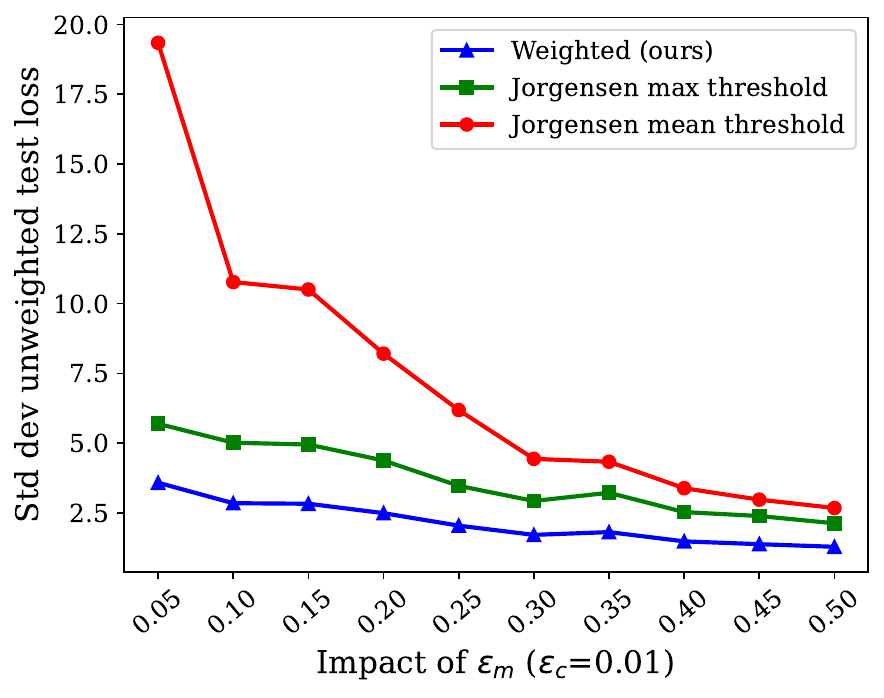}
        }
    \caption{Lower standard deviation compared to~\cite{jorgensen} on the synthetic dataset, while varying $\varepsilon_m$ (privacy level privacy level of the pragmatists), keeping $\varepsilon_c = 0.01, \varepsilon_l=1.0, f_c = 0.54, f_m = 0.37, f_l = 0.09$.}\label{fig:syn_jorgensen_epsm_std}
\end{figure*}

\begin{figure*}[!h]
    \centering
    \subfloat[Unregularized Loss ($\lambda = 1$)] {
        \includegraphics[width=0.4\textwidth]{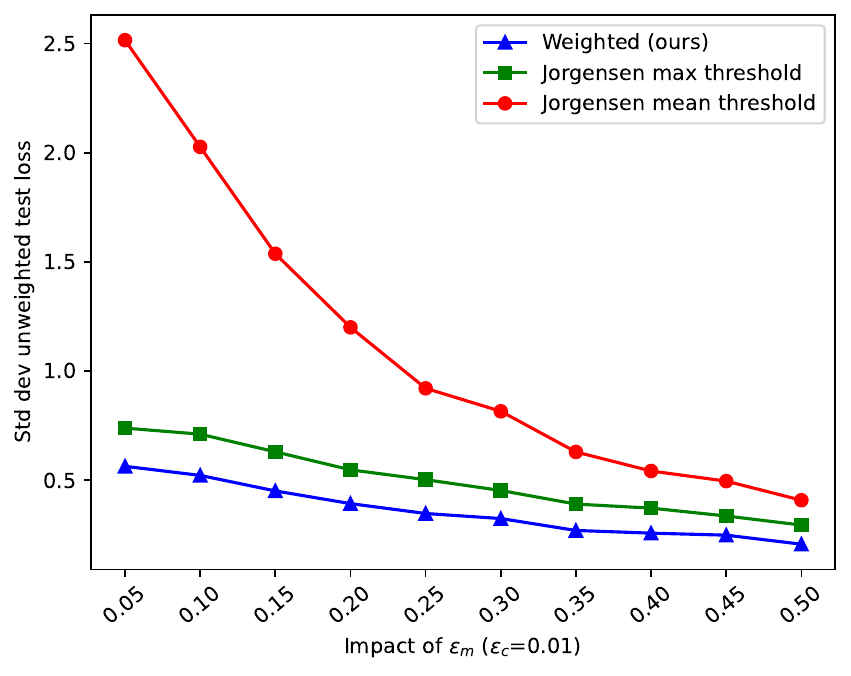}
    }
    \subfloat[Regularized Loss ($\lambda = 1$)]{
        \includegraphics[width=0.4\textwidth]{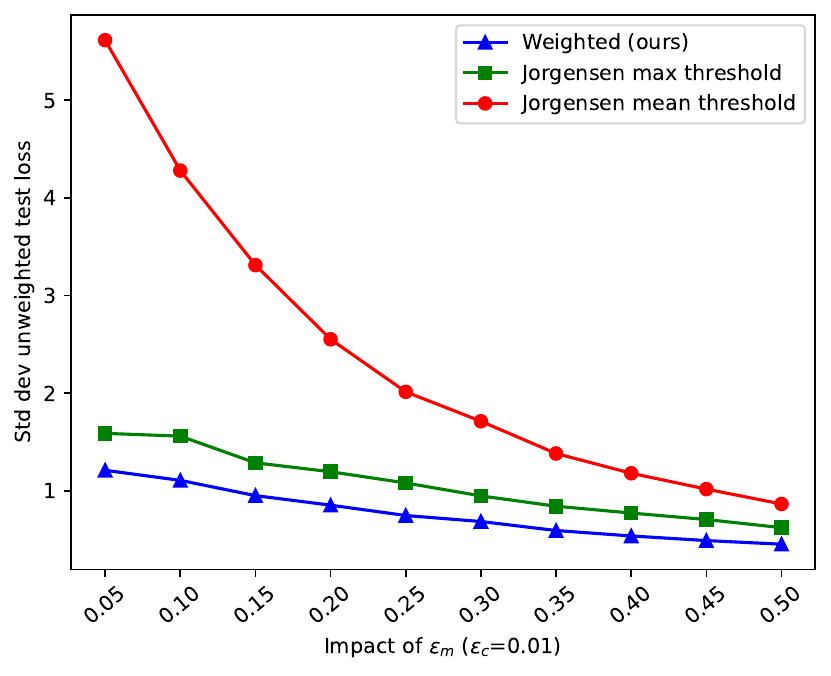}
        }
    \caption{Lower standard deviation compared to~\cite{jorgensen} on the Medical cost dataset, while varying $\varepsilon_m$ (privacy level privacy level of the pragmatists), keeping $\varepsilon_c = 0.01, \varepsilon_l=1.0, f_c = 0.54, f_m = 0.37, f_l = 0.09$.}\label{fig:insurance_jorgensen_epsm_std}
\end{figure*}

\begin{figure*}[!h]
    \centering
    \subfloat[Unregularized Loss ($\lambda = 100$)]{\includegraphics[width=0.4\textwidth]{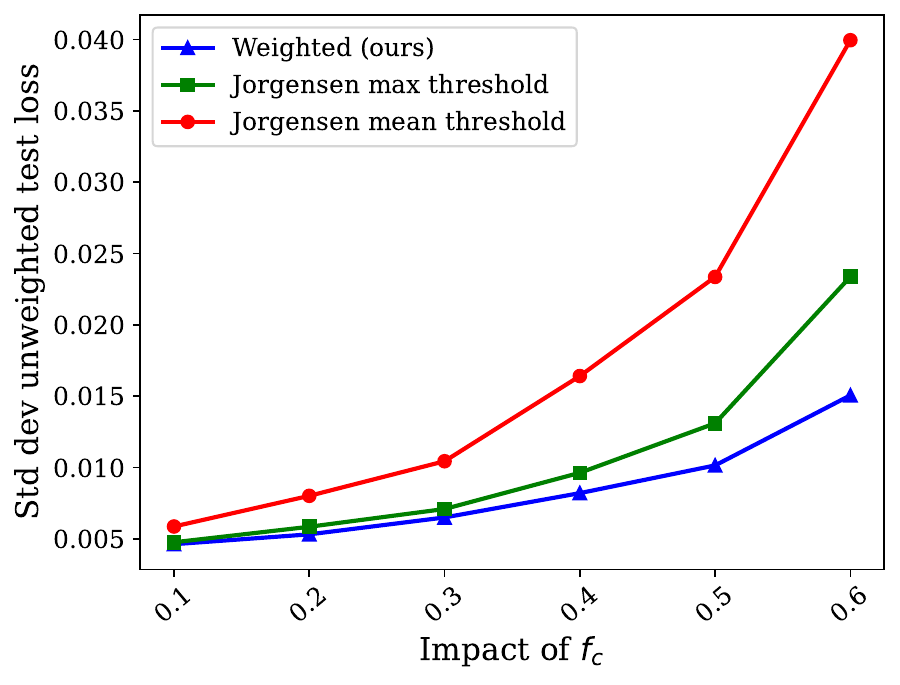}
    }
    \subfloat[Regularized Loss ($\lambda = 100$)]{
        \includegraphics[width=0.4\textwidth]{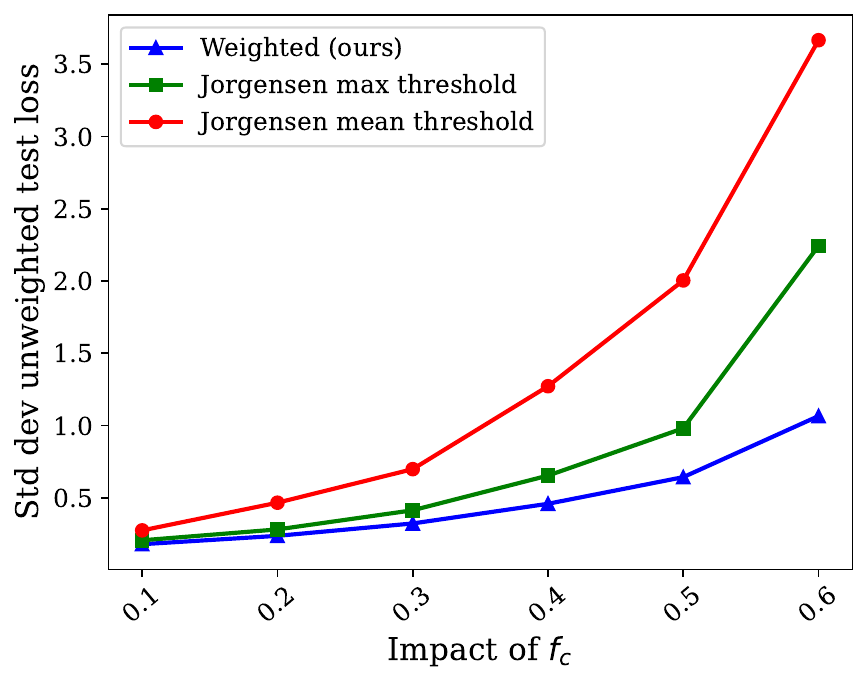}
    }
    \caption{Lower standard deviation compared to~\cite{jorgensen} on the synthetic dataset, while varying $f_c$ (the fraction of conservative users), keeping $\varepsilon_c=0.01, \varepsilon_m = 0.2, \varepsilon_l=1.0, f_m = 0.37, f_l = 1 - f_c - f_m = 0.63 - f_c$.}\label{fig:syn_jorgensen_fc_std}
\end{figure*}

\begin{figure*}[!h]
    \centering
    \subfloat[Unregularized Loss ($\lambda = 1$)]{\includegraphics[width=0.4\textwidth]{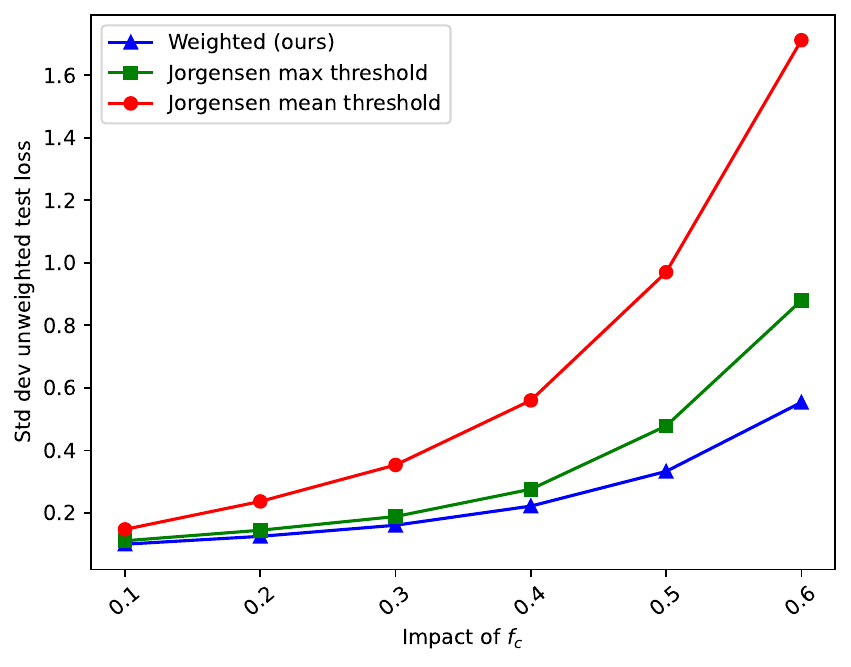
    }
    }
    \subfloat[Regularized Loss ($\lambda = 1$)]{
        \includegraphics[width=0.4\textwidth]{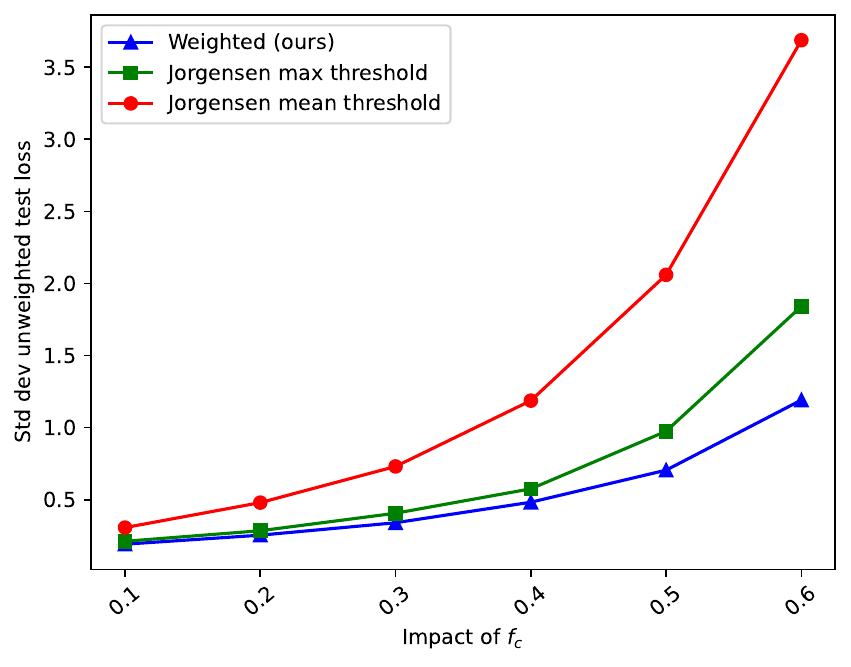}
    }
    \caption{Lower standard deviation compared to~\cite{jorgensen} on the Medical cost dataset, while varying $f_c$ (the fraction of conservative users), keeping $\varepsilon_c=0.01, \varepsilon_m = 0.2, \varepsilon_l=1.0, f_m = 0.37, f_l = 1 - f_c - f_m = 0.63 - f_c$.}\label{fig:insurance_jorgensen_fc_std}
\end{figure*}

\begin{figure*}[!h]
    \centering
    \subfloat[Unregularized Loss ($\lambda = 100$)]{
        \includegraphics[width=0.4\textwidth]{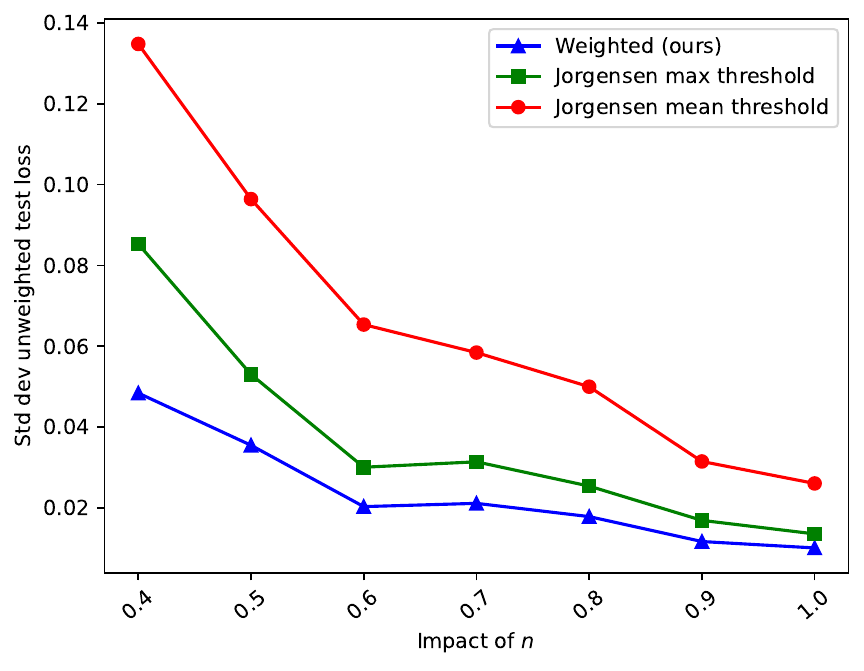}
    }
    \subfloat[Regularized Loss ($\lambda = 100$)]{
        \includegraphics[width=0.4\textwidth]{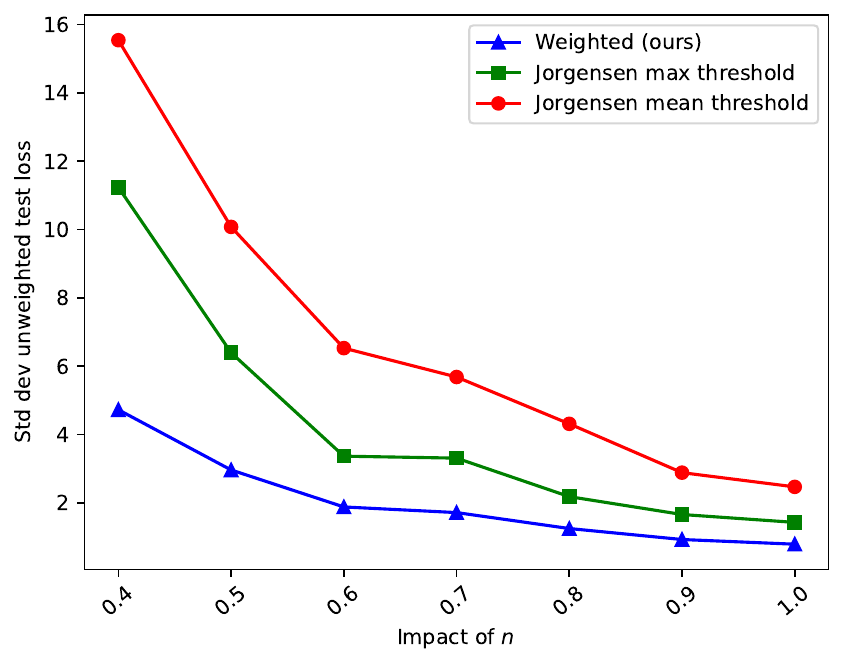}
    }
    \caption{Lower standard deviation compared to~\cite{jorgensen} on the synthetic dataset, while varying the parameter $n$ (the fraction of training samples used), keeping $\varepsilon_c=0.01, \varepsilon_m = 0.2, \varepsilon_l=1.0, f_c = 0.34, f_m = 0.43, f_l = 0.23$.}\label{fig:syn_jorgensen_n_std}
\end{figure*}

\begin{figure*}[!h]
    \centering
    \subfloat[Unregularized Loss ($\lambda = 1$)]{
        \includegraphics[width=0.4\textwidth]{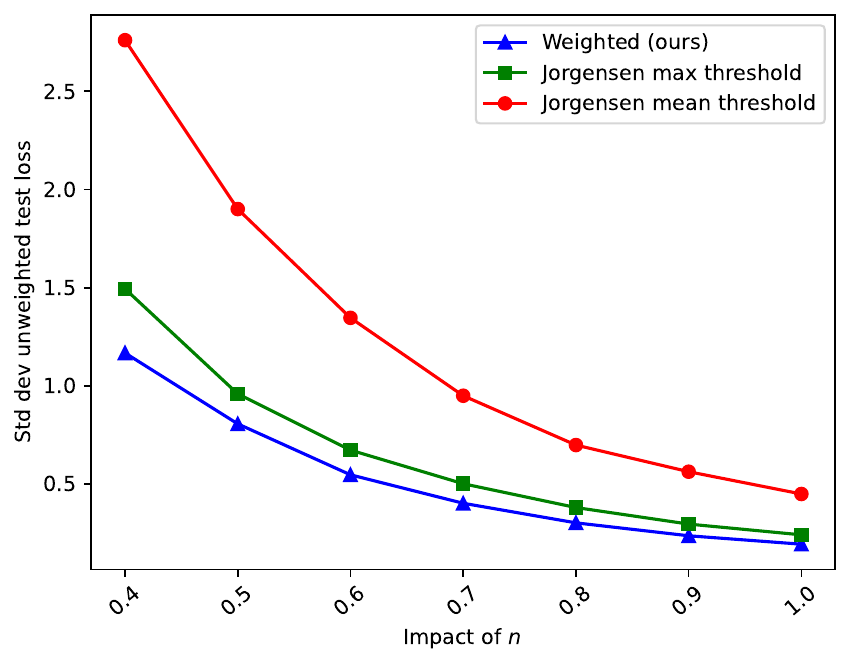}
    }
    \subfloat[Regularized Loss ($\lambda = 1$)]{
        \includegraphics[width=0.4\textwidth]{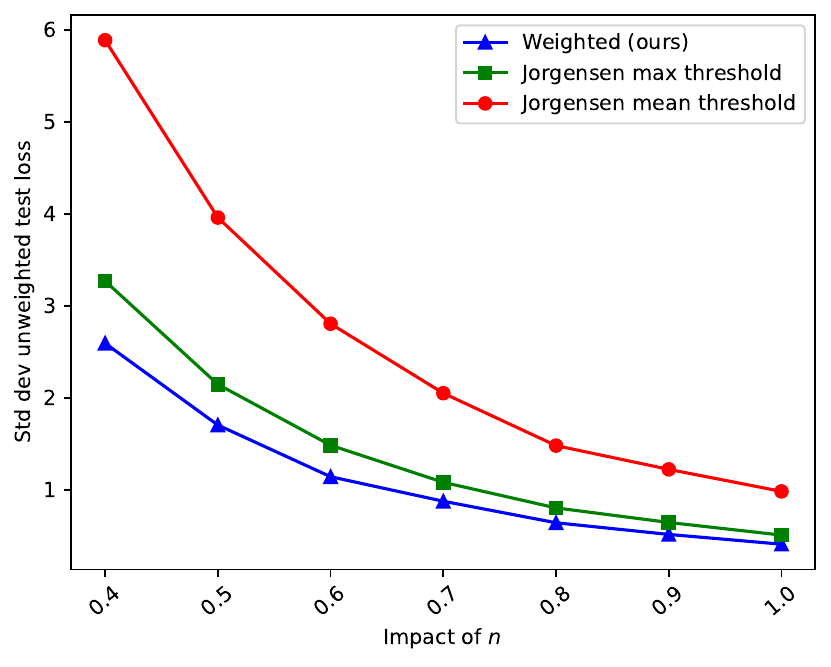}
    }
    \caption{Lower standard deviation compared to~\cite{jorgensen} on the Medical cost dataset, while varying the parameter $n$ (the fraction of training samples used), keeping $\varepsilon_c=0.01, \varepsilon_m = 0.2, \varepsilon_l=1.0, f_c = 0.34, f_m = 0.43, f_l = 0.23$.}\label{fig:insurance_jorgensen_n_std}
\end{figure*}

\begin{table*}[!h]
\centering
\resizebox{0.9\textwidth}{!}{%
\begin{tabular}{| c | ccc | ccc |}
\hline
\begin{tabular}[c]{@{}c@{}}Regularization\\ parameter\\ Lambda ($\lambda$)\end{tabular} &
  \begin{tabular}[c]{@{}c@{}}Unregularized\\ test loss \\ (\ours)\end{tabular} &
  \begin{tabular}[c]{@{}c@{}}Unregularized\\ test loss\\ (Jorgensen max)\end{tabular} &
  \begin{tabular}[c]{@{}c@{}}Unregularized\\ test loss\\ (Jorgensen mean)\end{tabular} &
  \begin{tabular}[c]{@{}c@{}}Regularized\\ test loss \\ (\ours)\end{tabular} &
  \begin{tabular}[c]{@{}c@{}}Regularized\\ test loss\\ (Jorgensen max)\end{tabular} &
  \begin{tabular}[c]{@{}c@{}}Regularized\\ test loss\\ (Jorgensen mean)\end{tabular} \\
\hline
1.00   & $\mathbf{9.69 \cross 10^{2}}$  & $1.34 \cross 10^{3}$  & $2.19 \cross 10^{3}$  & $\mathbf{1.53 \cross 10^{3}}$  & $2.01 \cross 10^{3}$  & $3.47 \cross 10^{3}$  \\
3.00   & $\mathbf{4.33 \cross 10^{1}}$  & $5.63 \cross 10^{1}$  & $9.98 \cross 10^{1}$  & $\mathbf{1.45 \cross 10^{2}}$  & $1.99 \cross 10^{2}$  & $3.39 \cross 10^{2}$  \\
5.00   & $\mathbf{1.13 \cross 10^{1}}$  & $1.41 \cross 10^{1}$  & $2.52 \cross 10^{1}$  & $\mathbf{5.65 \cross 10^{1}}$  & $7.69 \cross 10^{1}$  & $1.30 \cross 10^{2}$  \\
7.00   & $\mathbf{4.43}$                  & $5.50$                & $1.02 \cross 10^{1}$  & $\mathbf{3.15 \cross 10^{1}}$  & $4.20 \cross 10^{1}$  & $7.14 \cross 10^{1}$  \\
10.00  & $\mathbf{1.74}$                  & 2.18                  & 4.09                  & $\mathbf{1.72 \cross 10^{1}}$  & $2.32 \cross 10^{1}$  & $3.98 \cross 10^{1}$  \\
15.00  & $\mathbf{6.14 \cross 10^{-1}}$ & $7.69 \cross 10^{-1}$ & 1.30                  & $\mathbf{8.73}$         & $1.20 \cross 10^{1}$  & $2.03 \cross 10^{1}$  \\
20.00  & $\mathbf{2.88 \cross 10^{-1}}$ & $3.66 \cross 10^{-1}$ & $6.44 \cross 10^{-1}$ & $\mathbf{5.52}$                  & 7.43                  & $1.29 \cross 10^{1}$  \\
25.00  & $\mathbf{1.59 \cross 10^{-1}}$ & $2.15 \cross 10^{-1}$ & $3.79 \cross 10^{-1}$ & $\mathbf{3.91}$                  & 5.22                  & 9.15                  \\
50.00  & $\mathbf{2.37 \cross 10^{-2}}$ & $3.10 \cross 10^{-2}$ & $5.50 \cross 10^{-2}$ & $\mathbf{1.13}$                  & 1.53                  & 2.62                  \\
75.00  & $\mathbf{8.54 \cross 10^{-3}}$ & $1.11 \cross 10^{-2}$ & $1.88 \cross 10^{-2}$ & $\mathbf{5.73 \cross 10^{-1}}$ & $7.85 \cross 10^{-1}$ & 1.34                  \\
100.00 & $\mathbf{4.70 \cross 10^{-3}}$ & $5.64 \cross 10^{-3}$ & $9.85 \cross 10^{-3}$ & $\mathbf{3.72 \cross 10^{-1}}$ & $5.08 \cross 10^{-1}$ & $8.75 \cross 10^{-1}$ \\
200.00 & $\mathbf{1.33 \cross 10^{-3}}$ & $1.56 \cross 10^{-3}$ & $2.40 \cross 10^{-3}$ & $\mathbf{1.45 \cross 10^{-1}}$ & $1.99 \cross 10^{-1}$ & $3.41 \cross 10^{-1}$ \\
300.00 & $\mathbf{7.59 \cross 10^{-4}}$ & $8.62 \cross 10^{-4}$ & $1.26 \cross 10^{-3}$ & $\mathbf{8.77 \cross 10^{-2}}$ & $1.21 \cross 10^{-1}$ & $2.07 \cross 10^{-1}$ \\
400.00 & $\mathbf{5.41 \cross 10^{-4}}$ & $6.14 \cross 10^{-4}$ & $8.72 \cross 10^{-4}$ & $\mathbf{6.38 \cross 10^{-2}}$ & $8.48 \cross 10^{-2}$ & $1.48 \cross 10^{-1}$ \\
500.00 & $\mathbf{4.23 \cross 10^{-4}}$ & $4.79 \cross 10^{-4}$ & $6.45 \cross 10^{-4}$ & $\mathbf{4.92 \cross 10^{-2}}$ & $6.74 \cross 10^{-2}$ & $1.16 \cross 10^{-1}$ \\
\hline
\end{tabular}%
}
\caption{Improvements in variability of the test loss: Standard deviation of our algorithm is always lower, results on synthetic dataset with $d = 30,~ n = 100)$, while varying the regularization parameter $\lambda$, keeping $\varepsilon_c=0.01, \varepsilon_m = 0.2, \varepsilon_l=1.0, f_c = 0.34, f_m = 0.43$.}
\label{tab:stddev_syn_4_2_2_plevel_34_43_23}
\end{table*}

\begin{table}[!h]
\centering
\resizebox{0.9\textwidth}{!}{%
\begin{tabular}{|c|ccc|ccc|}
\hline
\begin{tabular}[c]{@{}c@{}}Regularization\\ parameter\\ Lambda ($\lambda$)\end{tabular} &
  \begin{tabular}[c]{@{}c@{}}Unregularized\\ test loss\\ (PDP-OP)\end{tabular} &
  \begin{tabular}[c]{@{}c@{}}Unregularized\\ test loss\\ (Jorgensen max)\end{tabular} &
  \begin{tabular}[c]{@{}c@{}}Unregularized\\ test loss\\ (Jorgensen mean)\end{tabular} &
  \begin{tabular}[c]{@{}c@{}}Regularized\\ test loss\\ (PDP-OP)\end{tabular} &
  \begin{tabular}[c]{@{}c@{}}Regularized\\ test loss\\ (Jorgensen max)\end{tabular} &
  \begin{tabular}[c]{@{}c@{}}Regularized\\ test loss\\ (Jorgensen mean)\end{tabular} \\
\hline
0.01 & $\mathbf{1.25 \cross 10^5}$     & $1.58 \cross 10^5$    & $3.12 \cross 10^5$    & $\mathbf{1.25 \cross 10^5}$    & $1.56 \cross 10^5$    & $3.01 \cross 10^5$    \\
0.05 & $\mathbf{1.08 \cross 10^3}$     & $1.32 \cross 10^3$    & $2.60 \cross 10^3$    & $\mathbf{1.13 \cross 10^3}$    & $1.37 \cross 10^3$    & $2.76 \cross 10^3$    \\
0.10 & $\mathbf{1.38 \cross 10^2}$     & $1.75 \cross 10^2$    & $3.28 \cross 10^2$    & $\mathbf{1.54 \cross 10^2}$    & $1.96 \cross 10^2$    & $3.75 \cross 10^2$    \\
0.50 & $\mathbf{1.35}$                  & 1.78                  & 3.26                  & $\mathbf{2.10}$                & 2.73                  & 4.99                  \\
0.60 & $\mathbf{8.33 \cross 10^{-1}}$  & 1.02                  & 1.96                  & $\mathbf{1.36}$                & 1.63                  & 3.21                  \\
0.70 & $\mathbf{5.31 \cross 10^{-1}}$  & $6.89 \cross 10^{-1}$ & 1.26                  & $\mathbf{9.54 \cross 10^{-1}}$ & 1.16                  & 2.25                  \\
0.80 & $\mathbf{3.64 \cross 10^{-1}}$  & $4.66 \cross 10^{-1}$ & $8.96 \cross 10^{-1}$ & $\mathbf{6.93 \cross 10^{-1}}$ & $9.05 \cross 10^{-1}$ & 1.73                  \\
0.90 & $\mathbf{2.67 \cross 10^{-1}}$  & $3.43 \cross 10^{-1}$ & $6.46 \cross 10^{-1}$ & $\mathbf{5.39 \cross 10^{-1}}$ & $6.82 \cross 10^{-1}$ & 1.28                  \\
1.00 & $\mathbf{1.98 \cross 10^{-1}}$  & $2.45 \cross 10^{-1}$ & $4.71 \cross 10^{-1}$ & $\mathbf{4.35 \cross 10^{-1}}$ & $5.27 \cross 10^{-1}$ & 1.00                  \\
2.00 & $\mathbf{3.96 \cross 10^{-2}}$  & $5.00 \cross 10^{-2}$ & $8.04 \cross 10^{-2}$ & $\mathbf{1.06 \cross 10^{-1}}$ & $1.28 \cross 10^{-1}$ & $2.43 \cross 10^{-1}$ \\
3.00 & $\mathbf{2.11 \cross 10^{-2}}$ & $2.33 \cross 10^{-2}$ & $3.84 \cross 10^{-2}$ & $\mathbf{4.80 \cross 10^{-2}}$ & $6.15 \cross 10^{-2}$ & $1.17 \cross 10^{-1}$ \\
5.00 & $\mathbf{1.11 \cross 10^{-2}}$  & $1.27 \cross 10^{-2}$ & $1.78 \cross 10^{-2}$ & $\mathbf{2.02 \cross 10^{-2}}$ & $2.50 \cross 10^{-2}$ & $4.77 \cross 10^{-2}$ \\
\hline
\end{tabular}%
}
\caption{Improvements in variability of the test loss: Standard deviation of our algorithm is always lower, results on Medical cost dataset, while varying the regularization parameter $\lambda$, keeping $\varepsilon_c=0.01, \varepsilon_m = 0.2, \varepsilon_l=1.0, f_c = 0.34, f_m = 0.43$.}
\label{tab:insurance_4_2_2_plevel_34_43_23}
\end{table}


\end{document}